\documentclass[sigconf]{acmart}

\newcommand{\newmethod}{\textit{DeepPipe}}

\usepackage{hyperref}
\usepackage{url}

\usepackage[utf8]{inputenc} 
\usepackage[T1]{fontenc}    
\usepackage{hyperref}       
\usepackage{url}            
\usepackage{booktabs}       
\usepackage{amsfonts}       
\usepackage{nicefrac}       
\usepackage{microtype}      
\usepackage{xcolor}         
\usepackage{subcaption}         
\usepackage{graphicx}
\usepackage{caption}
\usepackage[titlenumbered,ruled,linesnumbered, vlined]{algorithm2e}
\usepackage{multirow}

\newcommand{\dataset}{\mathcal{H}}

\usepackage{color}

\newcommand{\pipeline}{p}

\newcommand{\hp}{\lambda}
\newcommand{\pipeconf}{\pipeline_\hp}
\usepackage{float}
\usepackage{makecell}
\usepackage{changepage,threeparttable} 

\usepackage{amsmath}
\usepackage{wrapfig}
\usepackage[utf8]{inputenc}
\usepackage{pgfplots}
\DeclareUnicodeCharacter{2212}{−}
\usepgfplotslibrary{groupplots,dateplot}
\usetikzlibrary{patterns,shapes.arrows}
\pgfplotsset{compat=newest}
\usepackage{pgf}
\usepackage{mathtools}

\usepackage{amsmath,amsfonts,bm}

\def\eqref#1{equation~\ref{#1}}

\def\1{\bm{1}}

\DeclareMathAlphabet{\mathsfit}{\encodingdefault}{\sfdefault}{m}{sl}
\SetMathAlphabet{\mathsfit}{bold}{\encodingdefault}{\sfdefault}{bx}{n}

\newcommand{\R}{\mathbb{R}}

\DeclareMathOperator*{\argmax}{arg\,max}
\DeclareMathOperator*{\argmin}{arg\,min}

\usepackage{enumitem}
\usepackage{amsthm}
\usepackage{wrapfig}
\usepackage{pifont}
\newtheorem{theorem}{Theorem}[section]
\newtheorem{lemma}[theorem]{Lemma}
\newtheorem{proposition}[theorem]{Proposition}
\newtheorem{corollary}[theorem]{Corollary}
\newcommand{\cmark}{\ding{51}}%
\newcommand{\xmark}{\ding{55}}%

\AtBeginDocument{%
  \providecommand\BibTeX{{%
    \normalfont B\kern-0.5em{\scshape i\kern-0.25em b}\kern-0.8em\TeX}}}

\acmDOI{XXXXXXX.XXXXXXX}

\acmConference[KDD '23]{Make sure to enter the correct
  conference title from your rights confirmation emai}{August 06--10,
  2023}{Long Beach, CA}
\acmISBN{978-1-4503-XXXX-X/18/06}

\begin{document}

\copyrightyear{2023}
\acmYear{2023}
\setcopyright{acmlicensed}\acmConference[KDD '23]{Proceedings of the 29th ACM SIGKDD Conference on Knowledge Discovery and Data Mining}{August 6--10, 2023}{Long Beach, CA, USA}
\acmBooktitle{Proceedings of the 29th ACM SIGKDD Conference on Knowledge Discovery and Data Mining (KDD '23), August 6--10, 2023, Long Beach, CA, USA}
\acmPrice{15.00}
\acmDOI{10.1145/3580305.3599303}
\acmISBN{979-8-4007-0103-0/23/08}

\title{Deep Pipeline Embeddings for AutoML}



\author{Sebastian Pineda Arango}
\email{pineda@cs.uni-freiburg.de}
\affiliation{%
  \institution{Representation Learning Lab,\\ University of Freiburg}
  \city{Freiburg}
  \country{Germany}
}

\author{Josif Grabocka}
\email{grabocka@cs.uni-freiburg.de}
\affiliation{%
  \institution{Representation Learning Lab,\\ University of Freiburg}
  \city{Freiburg}
  \country{Germany}
}

\begin{abstract}

Automated Machine Learning (AutoML) is a promising direction for democratizing AI by automatically deploying Machine Learning systems with minimal human expertise. The core technical challenge behind AutoML is optimizing the pipelines of Machine Learning systems (e.g. the choice of preprocessing, augmentations, models, optimizers, etc.). Existing Pipeline Optimization techniques fail to explore deep interactions between pipeline stages/components. As a remedy, this paper proposes a novel neural architecture that captures the deep interaction between the components of a Machine Learning pipeline. We propose embedding pipelines into a latent representation through a novel per-component encoder mechanism. To search for optimal pipelines, such pipeline embeddings are used within deep-kernel Gaussian Process surrogates inside a Bayesian Optimization setup. Furthermore, we meta-learn the parameters of the pipeline embedding network using existing evaluations of pipelines on diverse collections of related datasets (a.k.a. meta-datasets). Through extensive experiments on three large-scale meta-datasets, we demonstrate that pipeline embeddings yield state-of-the-art results in Pipeline Optimization. 

\end{abstract}





\maketitle

\section{Introduction}

Machine Learning (ML) has proven to be successful in a wide range of tasks such as image classification, natural language processing, and time series forecasting. In a supervised learning setup practitioners need to design a sequence of choices comprising algorithms that transform the data (e.g. imputation, scaling) and produce an estimation (e.g. through a classifier or regressor). Unfortunately, manually configuring the design choices is a tedious and error-prone task. The field of AutoML aims at researching methods for automatically discovering the optimal design choices of ML pipelines \cite{he2021automl,hutter2019automated}. As a result, Pipeline Optimization~\citep{Olson2019_TPOT} or pipeline synthesis~\citep{Liu2020_AnADMM, Drori2021_AlphaD3M} is the primary open challenge of AutoML. 

\begin{figure}
    \centering
    \includegraphics[width=.85\linewidth]{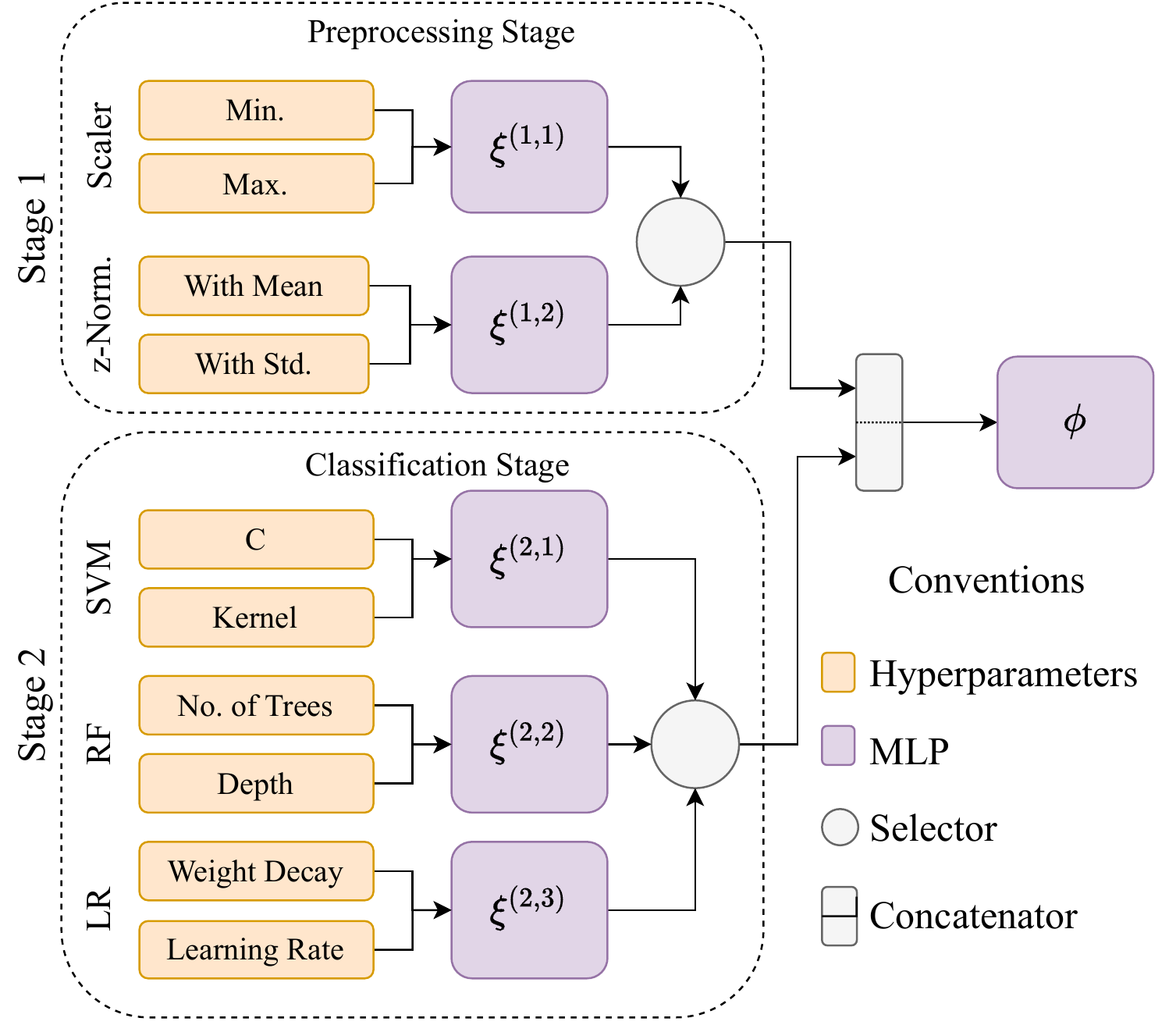}
    \caption{An example architecture for \textit{\newmethod{}} on a search space with two stages $\{\mathrm{Preprocessing}, \mathrm{Classification}\}$.}
    \label{fig:surrogate}
\end{figure}

Pipeline Optimization (PO) techniques need to capture the complex interaction between the algorithms of a Machine Learning (ML) pipeline and their hyperparameter configurations. Previous work demonstrates that the pipeline search can be automatized and achieve state-of-the-art predictive performance~\citep{feurer15_efficient, Olson2019_TPOT}. Some of these approaches include Evolutionary Algorithms~\citep{Olson2019_TPOT}, Reinforcement Learning~\citep{Rakotoarison2019_Automated, Drori2021_AlphaD3M} or Bayesian Optimization~\citep{feurer15_efficient, Thornton2012_AutoWeka, Alaa2018_AutoPrognosis}. Additionally, transfer learning has been shown to improve decisively PO by transferring efficient pipelines evaluated on other similar datasets~\citep{Fusi2018_Probabilistic, Yang2019_Oboe, Yang2020_AutoML}.

Unfortunately, no prior method uses Deep Learning to encapsulate the interaction between pipeline components. Existing techniques train traditional models as performance predictors on the concatenated hyperparameter space of all algorithms, such as Random Forests~\citep{feurer15_efficient}, or Gaussian Processes with additive kernels~\citep{Alaa2018_AutoPrognosis}. In this paper, we hypothesize that \textbf{we need Deep Learning, not only at the basic supervised learning level, but also at a meta-level for capturing the interaction between ML pipeline components} (e.g. the deep interactions of the hyperparameters of preprocessing, augmentation, and modeling stages).

As a result, we introduce \newmethod{}, a neural network architecture for embedding pipeline configurations on a latent space. Such deep representations are combined with Gaussian Processes (GP) for tuning pipelines with Bayesian Optimization (BO). 
We exploit the knowledge of the hierarchical search space of pipelines by mapping the hyperparameters of every algorithm through per-algorithm encoders to a hidden representation, followed by a fully connected network that receives the concatenated representations as input. An illustration of the mechanism is presented in Figure~\ref{fig:surrogate}. Additionally, we show that meta-learning this network through evaluations on auxiliary tasks improves the PO quality. Experiments on three large-scale meta-datasets show that our method achieves new state-of-the-art Pipeline Optimization.

\medskip

Our contributions are as follows:
\begin{itemize}
    \item We introduce \newmethod, a surrogate for BO that achieves state-of-the-art performance when optimizing a pipeline for a new dataset through transfer learning. 
    \item We present a novel and modular architecture that applies different encoders per stage and yields better generalization in low meta-data regimes, i.e. few/no auxiliary tasks.
    \item We conduct extensive evaluations against seven baselines on three large meta-datasets, and we further compare against rival methods in OpenML datasets to assess their performances under time constraints.
    \item We demonstrate that our pipeline representation helps achieve state-of-the-art results in optimizing pipelines for fine-tuning deep computer vision networks.
\end{itemize}

\section{Related Work}

\textbf{Hyperparameter Optimization  (HPO)} has been well studied over the past decade~ \citep{Bergstra2012_Random}. Techniques relying on Bayesian Optimization (BO) employ surrogates to approximate the response function of Machine Learning models, such as Gaussian Processes~\citep{Snoek2012_Practical}, Random Forests~\citep{Bergstra2011_Algorithms} or Bayesian neural networks~\citep{Snoek2015_DNGO, Springenberg2016_BOHAMIANN, wistuba2022supervising}. Further improvements have been achieved by applying transfer learning, where existing evaluations on auxiliary tasks help pre-training or meta-learning the surrogate. In this sense, some approaches use pre-trained neural networks with uncertainty outputs~\citep{Wistuba2021_FSBO, Perrone2018_Scalable,wei2021meta,khazideep2023}, or ensembles of Gaussian Processes~\citep{Feurer2018_RGPE}. 

\textbf{Deep Kernels} propose combining the benefits of stochastic models such as Gaussian Processes with neural networks~\citep{calandra2016manifold, garnelo2018neural, Wilson16_deep}. Follow-up work has applied this combination for training few-shot classifiers~\citep{patacchiola2020bayesian}. In the area of Hyperparameter Optimization, a successful option is to combine the output layer of a deep neural network with a Bayesian linear regression ~\citep{Snoek2015_DNGO}. Related studies \cite{Perrone2018_Scalable} extended this idea by pre-training the Bayesian network with auxiliary tasks. Recent work proposed using non-linear kernels, such as the Matérn kernel, on top of the pre-trained network to improve the performance of BO~\citep{Wistuba2021_FSBO, Wei2021_Meta}. However, to the best of our knowledge, we are the first to apply Deep Kernels for optimizing pipelines.

\textbf{Full Model Selection (FMS)} is also referred to as Combined Algorithm Selection and Hyperparameter optimization (CASH)~\citep{Hutter2019_Automated,feurer15_efficient}. FMS aims to find the best model and its respective hyperparameter configuration~\citep{Hutter2019_Automated}. A common approach is to use Bayesian Optimization with surrogates that can handle conditional hyperparameters, such as Random Forest~\citep{feurer15_efficient}, tree-structured Parzen estimators~\citep{Thornton2012_AutoWeka}, or ensembles of neural networks~\citep{SchillingJoint15}. 

\textbf{Pipeline Optimization (PO)} is a generalization of FMS where the goal is to find the algorithms and their hyperparameters for different stages of a Machine Learning Pipeline. Common approaches model the search space as a tree structure and use reinforcement learning~\citep{Rakotoarison2019_Automated, Drori2021_AlphaD3M, Guimaraes2017_RECIPE}, evolutionary algorithms~\citep{Olson2019_TPOT}, or Hierarchical Task Networks~\citep{Mohr2018_MLPlan} for searching pipelines. Other approaches use Multi-Armed Bandit strategies to optimize the pipeline, and combine them with Bayesian Optimization ~\citep{Swearingen2017_ATM} or multi-fidelity optimization~\citep{kishimoto2021bandit}. \citet{Alaa2018_AutoPrognosis} use additive kernels on a Gaussian Process surrogate to search pipelines with BO that groups the algorithms in clusters and fit their hyperparameters on independent Gaussian Processes, achieving an effectively lower dimensionality per input. By formulating the Pipeline Optimization as a constrained optimization problem, Liu et al~\citep{Liu2020_AnADMM} introduce an approach based on the alternating direction method of multipliers (ADMM)~\citep{gabay1976dual}.

\textbf{Transfer Learning for Pipeline Optimization and CASH} leverages information from previous (auxiliary) task evaluations. A few approaches use dataset meta-features to warm-start BO  with good configurations from other datasets~\cite {feurer15_efficient, Alaa2018_AutoPrognosis}. As extracting meta-features demands computational time, follow-up works find a portfolio based on these auxiliary tasks~\cite {feurer2020auto}. Another popular approach is to use collaborative filtering with a matrix of pipelines vs task evaluations to learn latent embeddings of pipelines. OBOE obtains the embeddings by applying a QR decomposition of the matrix on a time-constrained formulation~\citep{Yang2019_Oboe}. By recasting the matrix as a tensor, Tensor-OBOE~\citep{Yang2020_AutoML} finds latent representations via the Tucker decomposition. Furthermore, \citet{Fusi2018_Probabilistic} apply probabilistic matrix factorization for finding the latent pipeline representations. Subsequently, they use the latent representations as inputs for Gaussian Processes and explore the search space using BO. However, these methods using matrix factorization obtain latent representations of the pipelines that neglect the interactions of the hyperparameters between the pipeline's components.

\section{Preliminaries}

\subsection{Pipeline Optimization}

The pipeline of a ML system consists of a sequence of $N$ stages (e.g. dimensionality reducer, standardizer, encoder, estimator~\citep{Yang2020_AutoML}). At each stage $i \in \{1 \dots  N\}$ a pipeline includes one algorithm\footnote{AutoML systems might select multiple algorithms in a stage, however, our solution trivially generalizes by decomposing stages into new sub-stages with only a subset of algorithms.} from a set of $M_i$ choices (e.g. the \textit{estimator} stage can include the algorithms \{SVM, MLP, RF\}). Algorithms are tuned through their hyperparameter search spaces, where $\hp_{i, j}$ denotes the configuration of the $j$-th algorithm in the $i$-th stage. Furthermore, let us denote a pipeline $\pipeline$ as the set of indices for the selected algorithm at each stage, i.e. $\pipeline := \left(p_1,\dots,p_N \right)$, where $p_i \in \{1\dots M_i\}$ represents the index of the selected algorithm at the $i$-th pipeline stage. The hyperparameter configuration of a pipeline is the unified set of the configurations of all the algorithms in a pipeline, concretely $\hp{\left(\pipeline\right)} := \left( \hp_{1,p_1}, \dots, \hp_{N, p_{N}} \right)$, $\hp_{i, p_i} \in \Lambda_{i, p_i}$. Pipeline Optimization demands finding the optimal pipeline ${\pipeline}^*$ and its optimal configuration $\hp{\left(\pipeline^*\right)}$ by minimizing the validation loss of a trained pipeline on a dataset $\mathcal{D}$ as shown in Equation~\ref{eq:pipelineoptim}.

\begin{equation}
\label{eq:pipelineoptim}
    \left( {\pipeline}^*, {\hp{\left(\pipeline^*\right)}} \right) =
\operatorname*{arg\,min}_{ \substack{\pipeline \in \{1 \dots M_1 \} \times \dots \times \{1 \dots M_N\} , \\ \hp{\left(\pipeline\right)} \in  \Lambda_{1, p_{1}} \times \dots \times \Lambda_{N, p_{N}}  }} \mathcal{L}^{\text{val}}\big( \pipeline, \hp{\left(\pipeline\right)}, \mathcal{D} \big)
\end{equation}

From now we will use the term \textbf{pipeline configuration} for the combination of a sequence of algorithms $p$ and their hyperparameter configurations $\hp{\left(\pipeline\right)}$, and denote it simply as $\pipeconf := \left(p, \hp{\left(\pipeline\right)}\right)$. 

\subsection{Bayesian Optimization}
\label{sec:bopreliminary}
Bayesian optimization (BO) is a mainstream strategy for optimizing ML pipelines \citep{feurer15_efficient, Hutter2011_Sequential, Alaa2018_AutoPrognosis, Fusi2018_Probabilistic, SchillingJoint15}. Let us start with defining a history of $Q$ evaluated pipeline configurations as $\dataset=\{(\pipeconf^{(1)}, y^{(1)}), \dots, \linebreak (\pipeconf^{(Q)}, y^{(Q)})\}$, where $\smash{y^{(q)} \sim\mathcal{N}(f(\pipeconf^{(q)}), \sigma_q^2)}$ is a probabilistic modeling of the validation loss $f(\pipeconf^{(q)})$ achieved with the $q$-th evaluated pipeline configuration $\pipeconf^{(q)}, q \in \{1 \dots Q\}$. Such validation loss is approximated with a surrogate model, typically a Gaussian process (GP) regressor. We measure the similarity between pipelines via a kernel function $k: \text{dom}\left(\pipeconf\right) \times \text{dom}\left(\pipeconf\right) \rightarrow \mathbb{R}_{>0}$ parameterized with $\theta$, and represent similarities as a matrix $K^{\textcolor{blue}{'}}_{q,\ell} := k(\pipeconf^{(q)}, \pipeconf^{(\ell)}; \gamma), K^{\textcolor{blue}{'}} \in \mathbb{R}^{Q \times Q}_{>0}$. Since we consider noise, we define $K=K^{\textcolor{blue}{'}}+\sigma_pI$. A GP estimates the validation loss $f_{*}$ of a new pipeline configuration ${\pipeconf}^{(*)}$ by computing the posterior mean $\mathbb{E}\left[ f_{*}\right]$ and posterior variance $V\left[ f_{*}\right]$ as:

\begin{align}
\label{equation:posterior}
\begin{split}
\mathbb{E}\left[ f_{*} \; | \;{\pipeconf}^{(*)},\mathcal{H} \right] &= K_{*}^{T} K^{-1} y, \\ V\left[ f_{*} \; | \; {\pipeconf}^{(*)},\mathcal{H} \right] &= K_{**} - K_{*}^{T} K^{-1} K_{*}
\\ \text{where:    } K_{*,q} &= k\left({\pipeconf}^{(*)}, \pipeconf^{(q)}; \gamma \right), K_{*} \in \mathbb{R}^{Q}_{>0}
\\ K_{**} &=k\left({\pipeconf}^{(*)}, {\pipeconf}^{(*)}; \gamma \right), K_{**} \in \mathbb{R}_{>0}
\end{split}
\end{align}

We fit a surrogate iteratively using the observed configurations and their response in BO. Posteriorly, its probabilistic output is used to query the next configuration to evaluate by maximizing an acquisition function\citep{Snoek2012_Practical}. A common choice for the acquisition is Expected Improvement, defined as: 

\begin{equation}
\text{EI}(\pipeconf|\mathcal{H})=\mathbb{E}\left[\max\left\{ y_{\text{min} - \mu(\pipeconf)}, 0\right\}\right]
\end{equation}
where $y_{\text{min}}$ is the best-observed response in the history $\mathcal{H}$ and $\mu$ is the posterior of the mean predicted performance given by the surrogate, computed using Equation \ref{equation:posterior}. A common choice as a surrogate is Gaussian Processes, but for Pipeline Optimization we introduce \newmethod{}.

\section{\textit{Deep-Pipe}: BO with Deep Pipeline Configurations}


To apply BO to Pipeline Optimization (PO) we must define a kernel function that computes the similarity of pipeline configurations, i.e. $k\left(\pipeconf^{(q)}, \pipeconf^{(\ell)}; \theta\right) \;= \;?$. Prior work exploring BO for PO use kernel functions directly on the raw concatenated vector space of selected algorithms and their hyperparameters ~\citep{Alaa2018_AutoPrognosis} or use surrogates without dedicated kernels for the conditional search space~\citep{feurer15_efficient, Olson2019_TPOT, SchillingJoint15}.

However, we hypothesize that these approaches cannot capture the deep interaction between pipeline stages, between algorithms inside a stage, between algorithms across stages, and between different configurations of these algorithms. In order to address this issue we propose a simple, yet powerful solution to PO: learn a deep embedding of a pipeline configuration and apply BO with a deep kernel~\citep{Wistuba2021_FSBO, Wilson16_deep}. 

This is done by \newmethod{}, which searches pipelines in a latent space using BO with Gaussian Processes. We use a neural network $\phi(\pipeconf; \theta): \text{dom}(\pipeconf) \rightarrow \mathbb{R}^Z$ with weights $\theta$ to project a pipeline configuration to a $Z$-dimensional space. Then, we measure the pipelines' similarity in this latent space as $k\left( \phi(\pipeconf^{(q)}; \theta), \phi(\pipeconf^{(\ell)}; \theta)\right)$ using the popular Matérn 5/2 kernel \citep{marc2002classes}. Once we compute the parameters of the kernel similarity function, we can obtain the GP's posterior and conduct PO with BO as specified in Section~\ref{sec:bopreliminary}.


In this work, we exploit existing deep kernel learning machinery~\citep{Wistuba2021_FSBO, Wilson16_deep} to train the parameters $\theta$ of the pipeline embedding neural network $\phi$, and the parameters $\gamma$ of the kernel function $k$, by maximizing the log-likelihood of the observed validation losses $y$ of the evaluated pipeline configurations $\pipeconf$. The objective function for training a deep kernel is the log marginal likelihood of the Gaussian Process ~\citep{Rasmussen2006} with covariance matrix entries $k_{q, \ell}=k\left( \phi(\pipeconf^{(q)}; \theta), \phi(\pipeconf^{(\ell)}; \theta)\right)$.


\subsection{Pipeline Embedding Network}
\label{section:pipeline_embedding}

The main piece of the puzzle is: How to define the pipeline configuration embedding $\phi$? 

Our \textit{\newmethod{}} embedding is composed of two parts (i) per-algorithm neural network encoders, and (ii) a pipeline aggregation network. A visualization example of our DeepPipe embedding architecture is provided in Figure~\ref{fig:surrogate}. We define an encoder $\xi^{(i,j)}$ for the hyperparameter configurations of each $j$-th algorithm, in each $i$-th stage, as a plain multi-layer perceptron (MLP). Every encoder, parameterized by weights $\theta_{(i,j)}^{\text{enc}}$, maps the algorithms' configurations to a $L_i$-dimensional vector space:

\begin{align}
\begin{split}
    \vspace{1cm}
    \label{eq:encodersdef}
    \xi^{(i,j)}\left(\hp_{i, j}; \theta_{i,j}^{\text{enc}}\right) &= \text{MLP}\left(\hp_{i, j}; \theta_{i,j}^{\text{enc}}\right), \; \\
    \xi^{(i,j)}&: \Lambda_{i, j} \rightarrow \R^{L_i}
\end{split}
  \end{align}

For a pipeline configuration $\pipeconf$, represented with the indices of its algorithms $p$, and the configuration vectors of its algorithms $\lambda(p)$, we project all the pipeline's algorithms' configurations to their latent space using the algorithm-specific encoders. Then, we concatenate their latent encoder vectors, where our concatenation notation is $\R^{L_i} \oplus \R^{L_k} := \R^{L_i+L_k}$. Finally, the concatenated representation is embedded to a final $\R^Z$ space via an \textit{aggregation} MLP {$\psi$} with parameters $\theta^{\text{aggr}}$ as denoted below:

\begin{equation}
\begin{split}
    \label{eq:aggregationdef}
    \phi\left(p_\lambda \right) &:= {\psi}\left( \xi^{(1,p_1)}(\hp_{1, p_1}) \oplus \dots \oplus \xi^{(N,p_N)}(\hp_{N, p_N}) \; | \; \theta^{\text{aggr}} \right) \; \\
    {\psi} &: \R^{\sum_i L_i} \rightarrow \R^Z
\end{split}
\end{equation}

Within the $i$-th stage, only the output of one encoder is concatenated, therefore the output of the \textit{Selector} corresponds to the active algorithm in the $i$-th stage and can be formalized as $\xi^{(i,p_i)}(\hp_{i, p_i} )= \sum_{j=1}^{M_i} \mathbb{I}(j=p_i) \cdot \xi^{(i,j)}\left(\hp_{i, j}\right)$, where $\mathbb{I}$ denotes the indicator function.
Having defined the embedding $\phi$ in Equations~\ref{eq:encodersdef}-\ref{eq:aggregationdef}, we can plug it into the kernel function, optimize it minimizing the negative log-likelihood of the GP with respect to $\theta=\{\theta^{\text{enc}},\theta^{\text{aggr}}\}$, and conduct BO as in Section~\ref{sec:bopreliminary}. In Appendix \ref{appendix:interactions_discussion}, we discuss further how the different layers allow \newmethod{} to learn the interactions among components and stages. 

\begin{figure*}
    \includegraphics[width=.8\linewidth]{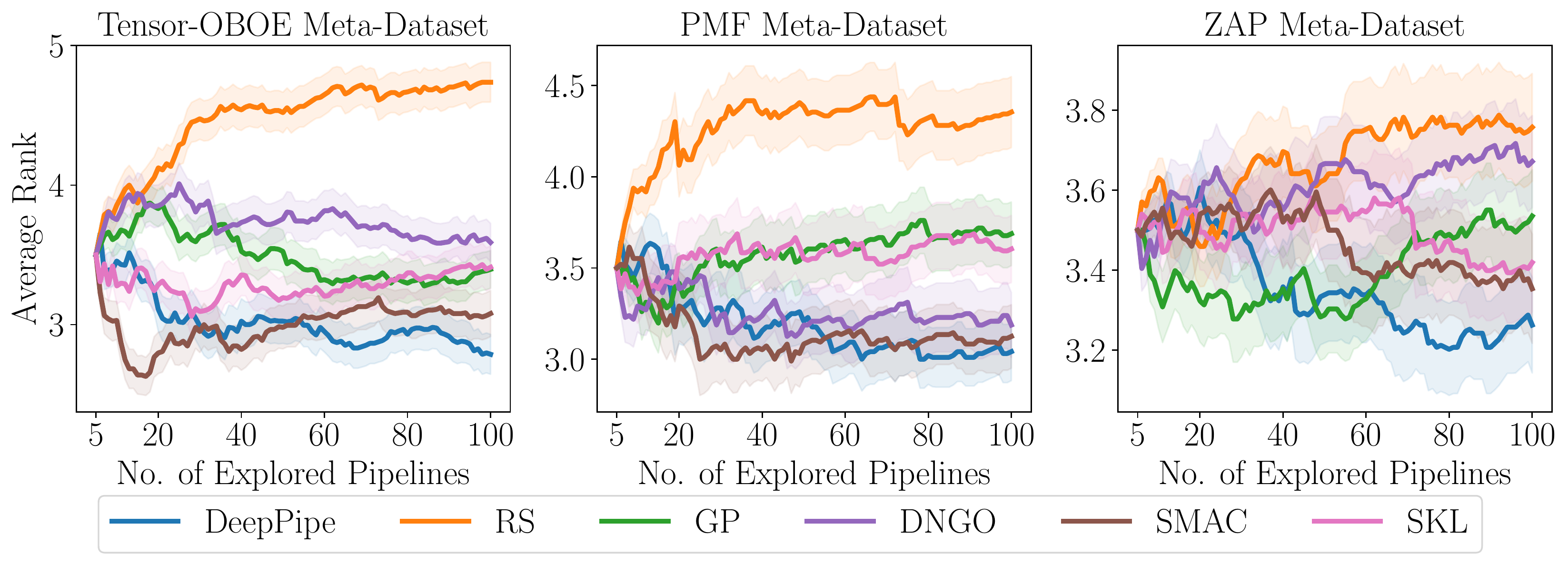}
    \caption{Comparison of \newmethod{} vs. standard PO methods (Experiment 1). Shaded lines indicate 95\% confidence interval.}
    \label{fig:non_transfer_results}
\end{figure*}

\begin{figure*}
    \includegraphics[width=.8\linewidth]{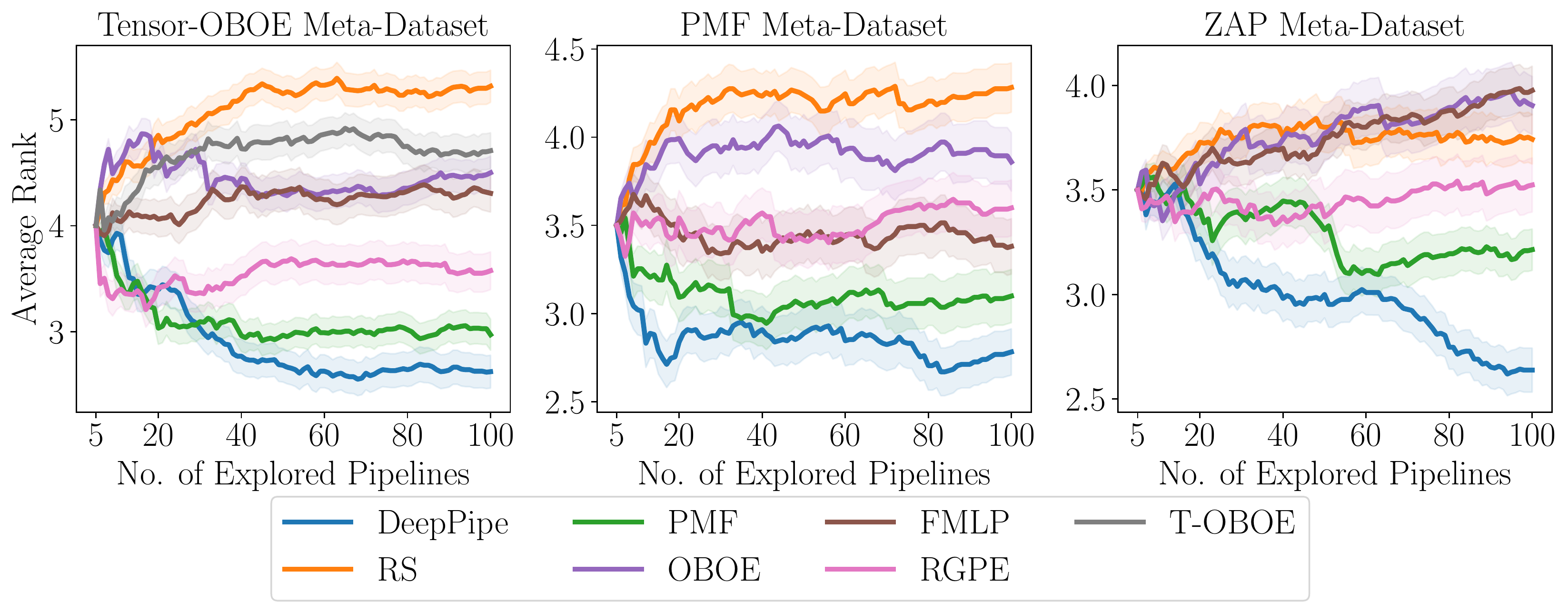}
    \caption{Comparison of \newmethod{} vs. transfer-learning PO methods (Experiment 2). Shaded lines indicate 95\% confidence interval.}
    \label{fig:transfer_results}
\end{figure*}

\subsection{Meta-learning our pipeline embedding}

In many practical applications, there exist computed evaluations of pipeline configurations on previous datasets, leading to the possibility of transfer learning for PO. Our \newmethod{} can be easily meta-learned from such past evaluations by pre-training the pipeline embedding network. Let us denote the meta-dataset of pipeline evaluations on $T$ datasets (a.k.a. auxiliary tasks) as $\dataset_t=\{({\pipeconf}^{(t,1)}, y^{(t,1)}), \dots, ({\pipeconf}^{(t,Q_t)}, y^{(t,Q_t)})\}, \, t \in \left\{1,\dots,T\right\}$, where $Q_t$ is the number of existing evaluations for the $t$-th dataset. As a result, we meta-learn our method's parameters to minimize the meta-learning objective of Equation~\ref{eq:metadeepgpobjective}.  This objective function corresponds to the negative log-likelihood of the Gaussian Processes using \newmethod{}'s extracted features as input to the kernel \citep{Wistuba2021_FSBO, patacchiola2020bayesian}. 

\begin{equation}
    \label{eq:metadeepgpobjective}
    \argmin_{\gamma, \theta} \; \sum_{t=1}^{T} \; {y^{(t)}}^{\mathrm{T}}{K^{(t)}}(\theta, \gamma)^{-1}{y^{(t)}}+\log\left|{K^{(t)}}(\theta, \gamma) \right|
\end{equation}

The learned parameters are used as initialization for the surrogate. We sample batches from the meta-training tasks and make gradient steps that maximize the marginal log-likelihood in Equation \ref{eq:metadeepgpobjective}, similar to previous work~\citep{Wistuba2021_FSBO}. The training algorithm for the surrogate is detailed in Algorithm \ref{alg:surrogate-training}. Every epoch, we perform the following operations for every task $t \in {1...T} $: (\textit{i)} Draw a set of $b$ observations (pipeline configuration and performance), (\textit{ii)} Compute the negative log marginal likelihood (our loss function) as in Equation \ref{eq:metadeepgpobjective},  (\textit{iii)}  compute the gradient of the loss with respect to the \newmethod{} parameters and (\textit{iv)} update \newmethod{} parameters. 
Additionally, we apply Early Convergence by monitoring the performance on the validation meta-dataset.

\begin{algorithm}[ht]
\SetAlgoLined
\KwIn{Learning rate $\eta$, meta-training data with $T$ tasks $\mathcal{H}=\bigcup_{t=1..T}\dataset^{(t)}$, number of epochs $ E$, batch size $b$}
\KwOut{Parameters  $\gamma$ and ${\theta}= \{\theta^{\mathrm{agg}}, \theta^{\mathrm{enc}}\}$}
Initialize $\gamma$ and $\theta$ at random\;
\For{$1,...,E$}
{
    \For{$t \in \{1,...,T\}$}
    {

        Sample batch $\mathcal{B}=\{(p_{\lambda}^{(t,i)},y^{(t,i)})\}_{i=1,\ldots,b}\sim\mathcal{H}^{(t)}$\;
        Compute negative log-likelihood $\mathcal{L}$ on $\mathcal{B}$. (Eq. \ref{eq:metadeepgpobjective})\;
        ${\theta^{\mathrm{agg}}}\leftarrow {\theta^{\mathrm{agg}}} - \eta \nabla_{{\theta^{\mathrm{agg}}}}\mathcal{L}$\;
        ${\theta^{\mathrm{enc}}}\leftarrow {\theta^{\mathrm{enc}}} - \eta \nabla_{{\theta^{\mathrm{enc}}}}\mathcal{L}$\;
        ${\gamma}\leftarrow {{\gamma}} - \eta\nabla_{{\gamma}}\mathcal{L}$\;

    }
}
\caption{\newmethod{} Meta-Training}
\label{alg:surrogate-training}
\end{algorithm}

 When a new pipeline is to be optimized on a new dataset (task), we apply BO (see Algorithm \ref{alg:surrogate-testing}). Every iteration we update the surrogate by fine-tuning the kernel parameters. However, the parameters of the MLP layers $\theta$ can be also optimized, as we did in Experiment 1, in which case the parameters were randomly initialized.  

\begin{algorithm}[t]
\SetAlgoLined
\KwIn{Learning rate $\eta$, initial observations $\mathcal{H} =\{(p_{\lambda}^{(i)},y^{(i)})\}_{i=1,\ldots,I} $, meta-learned surrogate with parameters $\theta$ and $\gamma $, number of surrogate updates $E_{Test}$, BO iterations $E_{BO}$, search space of pipelines $\mathcal{P}$, new task or dataset $\mathcal{D}$}
\KwOut{Pipeline Configuration $p_{\lambda}^*$}

\SetKwFunction{FT}{FineTune}
\SetKwFunction{FBO}{BO}
\SetKwProg{Fn}{Function}{:}{}

\Fn{\FT($\mathcal{H}, \gamma, \eta, E_{Test})$}{
\For{$1,...,E_{Test}$}
{
    Compute negative log-likelihood $\mathcal{L}$ on $\mathcal{D}$. (Objective function in Equation \ref{eq:metadeepgpobjective} with $T=1$)\;
    $\gamma \leftarrow \gamma - \eta\nabla_{\gamma} \mathcal{L}$\;
}
\KwRet{$\gamma$}
}
\Fn{\FBO{$\mathcal{H}, \eta, \theta, \gamma, E_{test}, E_{BO}$}}{

\For{$1,...,E$}
{
    $\gamma^{\prime} \leftarrow \FT( \mathcal{H}, \gamma, \eta, E_{Test}) $\;
    Compute $p^{*}_{\lambda} \in \argmax_{p_{\lambda} \in \mathcal{P}} \text{EI}(p_{\lambda}, \gamma^{\prime}, \theta)$ \;
    Observe performance $y^{*}$ of pipeline $p^{*}_{\lambda}$ \;
    Add new observation $\mathcal{H} \leftarrow \mathcal{H} \cup \{(p^{*}_{\lambda}, y^{*})\}$ \;
}
Compute best pipeline index $i_* \in \argmin_{i \in \{1 \dots |\mathcal{H}|\}} y_i$ \;
\KwRet{$p^{(i_*)}_{\lambda}$ } \;
}
\caption{Bayesian Optimization (BO) with \newmethod{}}
\label{alg:surrogate-testing}
\end{algorithm}

\section{Experiments}

\subsection{Meta-Datasets}
A meta-dataset is a collection of pipeline configurations and their respective performance evaluated in different tasks (i.e. datasets). In our experiments, we use the following meta-datasets.

\textbf{PMF} contains 38151 pipelines (after filtering out pipelines with only NaN entries), and 553 datasets~\citep{Fusi2018_Probabilistic}. Although not all the pipelines were evaluated in all tasks (or datasets), it still has a total of 16M evaluations. The pipeline search space has 2 stages (preprocessing and estimator) with 2 and 11 algorithms respectively.  Following the setup in the original paper~\citep{Fusi2018_Probabilistic}, we take 464 tasks for meta-training and 89 for meta-test. As the authors do not specify a validation meta-dataset, we sample randomly 15 tasks out of the meta-training dataset. 

\textbf{Tensor-OBOE} provides 23424 pipelines evaluated on 551 tasks~\citep{Yang2020_AutoML}. It contains 11M evaluations, as there exist sparse evaluations of pipelines and tasks. The pipelines include 5 stages: Imputator (1 algorithm), Dimensionality-Reducer (3 algorithms), Standardizer (1 algorithm), Encoder (1 algorithm), and Estimator (11 algorithms). We assign 331 tasks for meta-training, 110 for meta-validation, and 110 for meta-testing.  

\textbf{ZAP} is a benchmark that evaluates deep learning pipelines on fine-tuning state-of-the-art computer vision tasks~\citep{Ozturk2022ZeroShotAW}. The meta-dataset contains 275625 evaluated pipeline configurations on 525 datasets and 525 different Deep Learning pipelines (i.e. the best pipeline of a dataset was evaluated also on all other datasets). From the set of datasets, we use 315 for meta-training, 45 for meta-validation and 105 for meta-test, following the protocol of the original paper. 

In addition, we use \textbf{OpenML} datasets. It comprises 39 curated datasets~\citep{gijsbers2019open} and has been used in previous work for benchmarking~\citep{erickson2020autogluon}. Although this collection of datasets does not contain pipeline evaluations like the other three meta-datasets, we use it for evaluating the Pipeline Optimization in time-constrained settings~\citep{Ozturk2022ZeroShotAW}.

Information about the search space of every meta-dataset is clarified in Appendix~\ref{appendix:search_spaces}, and the splits of tasks per meta-dataset are found in Appendix~\ref{appendix:meta_dataset_splits}. All the tasks in the meta-datasets correspond to classification. We use the meta-training set for Pipeline Optimization (PO) methods using transfer learning or meta-learning, and the meta-validation set for tuning some of the hyper-parameters of the PO methods. Finally, we assess their performance on the meta-test set.

\paragraph{Meta-Datasets Preprocessing} We obtained the raw data for the meta-datasets from the raw repositories of PMF~\citep{PMFGit}
, TensorOBOE ~\citep{OboeGit} 
 and ZAP~\citep{ZAPGit}.
 PMF and ZAP repositories provide an accuracy matrix, while Tensor-OBOE specifies the error erorr. Moreover, the pipelines configurations are available in different formats for every meta-dataset, e.g. JSON or YAML. Therefore, we firstly convert all the configurations into a tabular format, and the performance matrices are converted to accuracies. Then, we proceed with the following steps: 1) One-Hot encode the categorical hyperparameters, 2) apply a log transformation $x_{new}=\text{ln}(x)$ to the hyperparameters whose value is greater than 3 standard deviations, 3) scale all the values to be in the range [0,1]. 
\subsection{Baselines}
\label{section:baselines}

We assess the performance of \newmethod{} by comparing it with the following set of baselines, which comprises transfer and non-transfer methods.

\textbf{Random Search (RS)} selects pipeline configurations by sampling randomly from the search space~\citep{Bergstra2012_Random}. 

\textbf{Probabilistic Matrix Factorization (PMF)} uses a surrogate model that learns shallow latent representation for every pipeline using the performance matrix of meta-training tasks~\citep{Fusi2018_Probabilistic}. We follow the setting for the original PMF for AutoML implementation~\citep{PMFGit}.

\textbf{OBOE} also uses matrix factorization for optimizing pipelines, but they aim to find fast and informative algorithms to initialize the matrix~\citep{Yang2019_Oboe}. We use the settings provided by the authors.

\textbf{Tensor-OBOE} formulates PO as a tensor factorization, where the rank of the tensor is equal to $1+N$, for $N$ being the number of stages in the pipeline~\citep{Yang2020_AutoML}. We use the setting provided by the original implementation \citep{Yang2019_Oboe}. We do not evaluate TensorOBOE on the ZAP and PMF meta-datasets because their performance matrix do not factorize into a tensor.

\textbf{Factorized Multilayer Perceptron (FMLP)} creates an ensemble of neural networks with a factorized layer~\citep{SchillingJoint15}. The inputs of the neural network are the one-hot encodings of the algorithms and datasets, in addition to the algorithms' hyperparameters. We use 100 networks with 5 neurons and ReLU activations as highlighted in the author's paper~\citep{SchillingJoint15}. 

\textbf{RGPE} builds an ensemble of Gaussian Processes using auxiliary tasks~\citep{Feurer2018_RGPE}. The ensemble weights the contributions of every base model and the new model fits the new task. We used the implementation from Botorch~\citep{Balandat2020_BoTorch}.

\textbf{Gaussian Processes (GP)} are a standard and strong baseline in hyperparameter optimization~\citep{Snoek2012_Practical}. In our experiments, we used Matérn 5/2 kernel.  

\textbf{DNGO} uses neural networks as basis functions with a Bayesian linear regressor at the output layer~\citep{Snoek2015_DNGO}. We use the implementation provided by \citet{Klein2020_pybnn}, and its default hyperparameters. 

\textbf{SMAC} uses Random Forest with 100 trees for predicting uncertainties~\citep{Hutter2011_Sequential}, with minimal samples leaf and split equal to 3. They have proven to handle well conditional search spaces ~\citep{feurer15_efficient}. 

\textbf{AutoPrognosis}~\citep{Alaa2018_AutoPrognosis} uses Structured Kernel Learning (SKL) and meta-learning for optimizing pipelines. We also compare AutoPrognosis against the meta-learned \newmethod{} by limiting the search space of classifiers to match the classifiers on the Tensor-OBOE meta-dataset \footnote{Specifically, the list of classifiers is: Random Forest, Extra Tree Classifier, Gradient Boosting", Logist Regression, MLP, linear SVM, kNN, Decision Trees, Adaboost, Bernoulli Naive Bayes, Gaussian Naive Bayes, Perceptron.}. Additionally, we compare \textbf{SKL} with our non-meta-learned \newmethod{} version using the default strategy for searching the additive kernels. For these experiments, we use the implementation in the respective author's repository \cite{AutoPrognosisGit}. 

\textbf{TPOT} is an AutoML system that conducts PO using evolutionary search~\citep{Olson2019_TPOT}. We use the original implementation but adopted the search space to fit the Tensor-OBOE meta-dataset (see Appendix \ref{appendix:search_spaces}). 

\begin{table}[]
\centering
\caption{Average Rank and Number of Observed Pipelines (\# Pips.) on OpenML Datasets after Experiment 3.}\label{table:time_constrained_experiment}
\begin{tabular}{c cccc}
\toprule
\multirow{2}{*}{\textbf{Method}}                                 & \multicolumn{2}{c}{\textbf{10 Mins.}}                               & \multicolumn{2}{c}{\textbf{1 Hour}} \\\cmidrule(ll){2-3} \cmidrule(ll){4-5}
                                                        & \multicolumn{1}{c}{\textbf{ Rank}}  & \textbf{\# Pips.}  & \multicolumn{1}{c}{\textbf{Rank}} & \textbf{\# Pips.} \\ \hline
TPOT                                            & \multicolumn{1}{c}{3.20 $\pm$ 0.19}   & 45 $\pm$ 46   & \multicolumn{1}{c}{3.35 $\pm$ 0.19}  & 70 $\pm$ 41  \\ 
T-OBOE                                           & \multicolumn{1}{c}{4.38 $\pm$ 0.17} & 84 $\pm$ 57 & \multicolumn{1}{c}{4.36 $\pm$ 0.20} & 178 $\pm$ 69 \\ 
OBOE                                           & \multicolumn{1}{c}{3.99$\pm$ 0.19} & 120 $\pm$ 70 & \multicolumn{1}{c}{4.08 $\pm$ 0.21} & 467 $\pm$ 330 \\ 
SMAC                                           & \multicolumn{1}{c}{3.24$\pm$ 0.16} & 81 $\pm$ 115 & \multicolumn{1}{c}{3.16 $\pm$ 0.14} & 452 $\pm$ 637 \\ 
PMF                                           & \multicolumn{1}{c}{3.04$\pm$ 0.15} & 126 $\pm$ 197 & \multicolumn{1}{c}{2.93 $\pm$ 0.15} & 523 $\pm$ 663 \\ 
\newmethod{}                                          & \multicolumn{1}{c}{\textbf{2.74 $\pm$ 0.12}}  & 94 $\pm$ 128 & \multicolumn{1}{c}{\textbf{2.89 $\pm$ 0.13}} & 356 $\pm$ 379 \\ \bottomrule
\end{tabular}
\end{table}

\begin{table}[]
\caption{{Comparison AutoPrognosis (AP) vs  \newmethod{} (DP)} }\label{tab:comparison_with_autoprognosis}
\begin{tabular}{ccccc}

\multicolumn{1}{l}{}    $\mathbf{E_{BO}}$      &       \textbf{Alg.}                     & \textbf{Rank} & \textbf{ Acc.} & \textbf{ Time (Min.)} \\ \toprule
\multirow{2}{*}{\textbf{50}}  & \textbf{AP} & 1.558   $\pm$ 0.441             & 0.863   $\pm$ 0.114             & 161    $\pm$ 105             \\
                              & \textbf{DP}          & \textbf{ 1.441   $\pm$ 0.441 }            &  \textbf{0.869   $\pm$ 0.111 }             & 15   $\pm$ 25              \\ \midrule
\multirow{2}{*}{\textbf{100}} & \textbf{AP} & 1.513           $\pm$ 0.469             & 0.871  $\pm$ 0.095             & 308    $\pm$ 186             \\
                              & \textbf{DP}          & \textbf{ 1.486   $\pm$ 0.469 }             &  \textbf{0.873  $\pm$ 0.097}             & 37    $\pm$ 90  \\  \bottomrule 
                              
\end{tabular}
\end{table}

\begin{table*}[t]\centering
\caption{Average rank among \newmethod{} variants for newly-added algorithms (Tensor-OBOE)}\label{tab:omit_estimator_oboe }
\begin{tabular}{ccccccccccccccc}\toprule
\multirow{2}{*}{Enc.} &\multirow{2}{*}{MTd.} &\multicolumn{2}{c}{Omitted in}
&\multicolumn{10}{c}{Omitted Estimator} \\\cmidrule(ll){3-4} \cmidrule(ll){5-14}
& &MTr. &MTe. &ET &GBT &Logit &MLP &RF &lSVM &KNN &DT &AB &GB/PE \\\midrule
\cmark  & \cmark  & \cmark  & \cmark &3.2398 &3.1572 &3.0503 &3.1982 &3.4135 &3.3589 &3.2646 &3.2863 &3.1580 &3.3117 \\
\cmark  & \xmark & \cmark  &  \xmark&3.5319 &3.0934 &3.6362 &3.4780 &3.4712 &3.3829 &3.6312 &3.3691 &3.6333 &3.4642 \\
\cmark  &\cmark  & \xmark & \xmark &\textbf{2.5582} &\textbf{2.6773} & \textbf{2.7086} & \textbf{2.5761} & \textbf{2.6485} &\textbf{2.6938} & \textbf{2.6812} & \textbf{2.5596} & \textbf{2.5936} & \textbf{2.5546} \\
\xmark &\cmark &\cmark & \xmark &2.9247 &3.0743 & 2.8802 &3.0423 & \underline{2.6691} &2.8026 &2.7408 &   2.9161  &2.9214 &2.8689 \\
\cmark & \cmark & \cmark & \xmark &\underline{2.7455} & \underline{2.9978} &\underline{2.7248} &\underline{2.7054} & 2.7978 & \underline{2.7619} & \underline{2.6822} & \underline{2.8688} &\underline{2.6938} & \underline{2.8007} \\
\bottomrule
\end{tabular}
\end{table*}

\subsection{Experimental Setup for \newmethod{} }
\label{section:experimental_setup}

The encoders and the aggregation layers are Multilayer Perceptrons with ReLU activations. We keep an architecture that is proportional to the input size, such that the number of neurons in the hidden layers for the encoder of algorithm $j$-th in $i$-th stage with $|\Lambda_{i,j}|$ hyperparameters is $F\cdot|\Lambda_{i,j}|$, given an integer factor $F$. The output dimension of the encoders of the $i$-th stage is defined as $L_{i} = \mathrm{max}_j|\Lambda_{i,j}|$. The number of total layers (i.e. encoder and aggregation layers) is fixed to 4 in all experiments, thus the number of encoders $\ell_e$ is chosen from \{0,1,2\} and the number of aggregation layers is set to $4-\ell_e$. The specific values of the encoders' input dimensions are detailed in Appendix~\ref{appendix:search_spaces}. We choose $F \in \{4,6,8,10\}$ based on the performance in the validation split. Accordingly, we use the following values for \newmethod{}: \textit{(i)} in Experiment 1: 1 encoder layer (all meta-datasets), $F=6$ (PMF and ZAP) and $F=8$ (Tensor-OBOE), \textit{(ii)} in Experiment 2: $F=8$, no encoder layer (PMF, Tensor-OBOE) and one encoder layer (ZAP), \textit{(iii)} in Experiment 3: $F=8$ and no encoder layers, \textit{(iv)} in Experiment 4 we use $F=8$ and $\{0,1\}$ encoder layers. Finally \textit{(iv)} in Experiment 5 we use  $F=8$ and $\{0,1,2\}$ encoder layers. 

In all experiments (except Experiment 1), we meta-train the surrogate following Algorithm \ref{alg:surrogate-training} for 10000 epochs with the Adam optimizer and a learning rate of $10^{-4}$, batch size $1000$, and the Matérn kernel for the Gaussian Process. During meta-testing, when we perform BO to search for a pipeline, we fine-tune only the kernel parameters $\gamma$ for 100 gradient steps. In the non-transfer experiments (Experiment 1) we tuned the whole network for $10000$ iterations, while the rest of the training settings are similar to the transfer experiments. In Experiment 5 we fine-tune the whole network for 100 steps when no encoders are used. Otherwise, we fine-tune only the encoder associated with the omitted estimator and freeze the rest of the network. We ran all experiments on a CPU cluster, where each node contains two Intel Xeon E5-2630v4 CPUs with 20 CPU cores each, running at 2.2 GHz. We reserved a total maximum memory of 16GB.  
Further details on the architectures for each search space are specified in Appendix~\ref{appendix:architecture_details}. Finally, we use the Expected Improvement as an acquisition function for \newmethod{} and all the baselines.

\paragraph{Initial Configurations}
All the baselines use the same five initial configurations, i.e. $I=5$ in Algorithm \ref{alg:surrogate-testing}. For the experiments with the PMF-Dataset, we choose these configurations with the same procedure as the authors~\cite{Fusi2018_Probabilistic}, where they use dataset meta-features to find the most similar auxiliary task to initialize the search on the test task. Since we do not have meta-features for the Tensor-OBOE meta-dataset, we follow a greedy initialization approach~\cite{Metz2020_Thousand}. This was also applied to the ZAP-Dataset. Specifically, we select the best-performing pipeline configuration by ranking their performances on the meta-training tasks. Subsequently, we iteratively choose four additional configurations that minimize $\sum_{t \in \mathrm{Tasks}} \hat{r}_{t}$, where $\hat{r}_t=\min_{\pipeline \in \mathcal{X}} r_{t,\pipeline}$, given that $r_{t,\pipeline}$ is the rank of the pipeline $\pipeline$ on task $t$.
Additional details on the setup can be found in  our source code\footnote{The code is available in this repository: \url{https://github.com/releaunifreiburg/DeepPipe}}.

\subsection{Research Hypotheses and Associated Experiments}
\label{section:hypothesis}

We describe the different hypotheses and experiments for testing the performance of \newmethod{}.

\textbf{Hypothesis 1:} \newmethod{} outperforms standard PO baselines.

\textbf{Experiment 1:} We evaluate the performance of \newmethod{} when no meta-training data is available. We compare against four baselines: Random Search (RS)~\citep{Bergstra2011_Algorithms}, Gaussian Processes (GP)~\citep{Rasmussen2006}, DNGO~\citep{Snoek2015_DNGO}, SMAC~\citep{Hutter2011_Sequential} and SKL~\citep{Alaa2018_AutoPrognosis}. We evaluate their performances on the aforementioned PMF, Tensor-OBOE and ZAP meta-datasets. In Experiments 1 and 2 (below), we select 5 initial observations to warm-start the BO, then we run 95 additional iterations. 

\textbf{Hypothesis 2:} Our meta-learned \newmethod{} outperforms state-of-the-art transfer-learning PO methods.

\textbf{Experiment 2:} We compare our proposed method against baselines that use auxiliary tasks (a.k.a. meta-training data) for improving the performance of Pipeline Optimization: Probabilistic Matrix Factorization (PMF) \citep{Fusi2018_Probabilistic}, Factorized Multilayer Perceptron (FMLP) \citep{SchillingJoint15}, OBOE \citep{Yang2019_Oboe} and Tensor OBOE \citep{Yang2020_AutoML}. 
Moreover, we compare to RGPE \citep{Feurer2018_RGPE}, an effective baseline for transfer HPO~\citep{Pineda2021_HPOB}. We evaluate the performances on the PMF and Tensor-OBOE meta-datasets.

\textbf{Hypothesis 3:} \newmethod{} leads to strong any-time results in a time-constrained PO problem.

\textbf{Experiment 3:} Oftentimes practitioners need AutoML systems that discover efficient pipelines within a small time budget. To test the convergence speed of our PO method we ran it on the aforementioned OpenML datasets for a budget of 10 minutes, and also 1 hour. We compare against five baselines: \textit{(i)} TPOT~\citep{Olson2019_TPOT} adapted to the search space of Tensor-OBOE (see Appendix~\ref{appendix:search_spaces}), \textit{(ii)} OBOE and Tensor-OBOE \citep{Yang2019_Oboe,Yang2020_AutoML} using the time-constrained version provided by the authors, \textit{(iii)} SMAC \citep{Hutter2011_Sequential}, and \textit{(iv)} PMF \citep{Fusi2018_Probabilistic}. The last three had the same five initial configurations used to warm-start BO as detailed in Experiment 1. Moreover, they were pre-trained with the Tensor-OBOE meta-dataset and all the method-specific settings are the same as in Experiment 2. We also compared \newmethod{} execution time with AutoPrognosis~\citep{imrie2022autoprognosis}, and report the performances after 50 and 100 BO iterations. 

\textbf{Hypothesis 4:} Our novel encoder layers of \newmethod{} enable an efficient PO when the pipeline search space changes, i.e. when developers add a new algorithm to an ML system.

\textbf{Experiment 4:} 
A major obstacle to meta-learning PO solutions is that they do not generalize when the search space changes, especially when the developers of ML systems add new algorithms. Our architecture quickly adapts to newly added algorithms \textbf{because only an encoder sub-network for the new algorithm should be trained}. To test the scenario, we ablate the performance of five versions of \newmethod{} and try different settings when we remove a specific algorithm (an estimator) either from meta-training, meta-testing, or both.

\textbf{Hypothesis 5:} The encoders in \newmethod{} introduce an inductive bias where latent representation vectors of an algorithm's configurations are co-located and located distantly from the representations of other algorithms' configurations. Formally, given three pipelines $\pipeline^{(l)}, \pipeline^{(m)}, \pipeline^{(n)}  $ if $p_i^{(l)} =p_i^{(m)} , p_i^{(l)} \neq p_i^{(n)} $ then $||\phi(p^{(l)})-\phi(p^{(m)})|| < ||\phi(p^{(m)})-\phi(p^{(n)})||$ with higher probability when using encoder layers, given that $p_i^{(n)}$ is the index of the algorithm in the $i$-th stage. Furthermore, we hypothesize that the less number of tasks during pre-training, the more necessary this inductive bias is.

\textbf{Experiment 5:} We sample 2000 pipelines of 5 estimation algorithms on the TensorOBOE dataset. Subsequently, we embed the pipelines using a \newmethod{} with 0, 1, and 2 encoder layers, and weights $\theta$, initialized such that $\theta_i \in \theta$ are independently identically distributed $\theta_i \sim \mathcal{N}(0, 1)$. Finally, we visualize the embeddings with T-SNE~\citep{van2008visualizing} and compute a cluster metric to assess how close pipelines with the same algorithm are in the latent space: $\mathbb{E}_{\pipeline^{(l)}, \pipeline^{(m)}, \pipeline^{(n)}}(\mathbb{I}(||\phi(p^{(l)})-\phi(p^{(m)})|| < ||\phi(p^{(m)})-\phi(p^{(n)})||)).$ To test the importance of the inductive bias vs the number of pre-training tasks, we ablate the performance of \newmethod{} for different percentages of pre-training tasks (0.5\%, 1\%, 5\%, 10\%, 50\%, 100\%) under different values of encoder layers.

    

\section{Results}

\label{section:results}

\begin{figure*}
     \centering
    \includegraphics[width=0.85\linewidth]{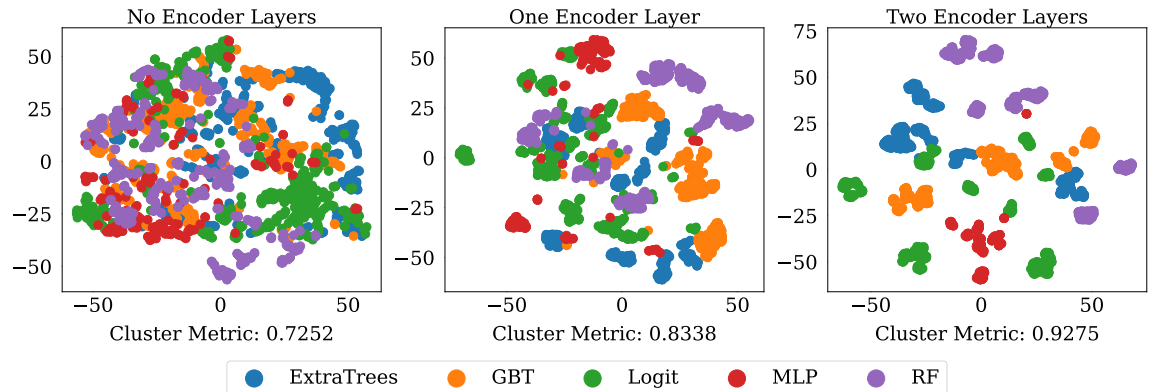}
    \caption{Embeddings of Pipelines produced by a random initialized \newmethod{} (after applying T-SNE). The color indicates the active algorithm in the Estimation stage of Tensor-OBOE Meta-Dataset.}
    \label{fig:random_embeddings}
\end{figure*}

We present the results for \textbf{Experiments 1 and 2} in Figures \ref{fig:non_transfer_results} and \ref{fig:transfer_results}, respectively. In both cases, we compute the ranks of the classification accuracy achieved by the discovered pipelines of each technique, averaged across the meta-testing datasets. The shadowed lines correspond to the 95\% confidence intervals. Additional results showing the mean regret are included in Appendix~\ref{appendix:additional_results}.  In Experiment 1 (standard/non-transfer PO) \newmethod{} achieved the best performance for both meta-datasets, whereas SMAC attained the second place.

In \textbf{Experiment 2} \newmethod{} strongly outperforms all the transfer-learning PO baselines in all meta-datasets. Given that \newmethod{} yields state-of-the-art PO results on both standard and transfer-learning setups, we conclude that our pipeline embedding network computes efficient representations for PO with Bayesian Optimization. In particular, the results on the ZAP meta-dataset indicate the efficiency of \newmethod{} in discovering state-of-the-art Deep Learning pipelines for computer vision. We discuss additional ablations and comparisons in Appendix \ref{appendix:additional_results}.

\textbf{Experiment 3} conducted on the OpenML datasets shows that \newmethod{} performs well under restricted budgets, as reported in Table \ref{table:time_constrained_experiment}. We present the values for the average rank and the average number of observed pipelines after 10 and 60 minutes. Additionally, Table \ref{tab:comparison_with_autoprognosis} shows the number of pipelines observed by AutoPrognosis and \newmethod{} during the execution, demonstrating that \newmethod{} manages to explore a relatively high number of pipelines while attaining the best performance. Although our method does not incorporate any direct way to handle time constraints, it outperforms other methods that include heuristics for handling a quick convergence, such as OBOE and Tensor-OBOE.


Additionally, we compare \newmethod{} with the AutoPrognosis 2.0 library~\citep{imrie2022autoprognosis} on the Open ML datasets, where we run both methods for 50 and 100 BO iterations ($E_{BO}$). We report the average and standard deviation for rank, accuracy, and time. \newmethod{} achieves the best average rank, i.e. a lower average rank than AutoPrognosis. This is complemented by having the highest average accuracy. Interestingly, our method is approximately one order of magnitude faster than AutoPrognosis. We note this is due to the time overhead introduced by their Gibbs sampling strategy for optimizing the structured kernel, whereas our approach uses gradient-based optimization.

Furthermore, the results reported in Tables \ref{tab:omit_estimator_oboe } and \ref{tab:omit_estimator_pmf } for \textbf{Experiment 4} indicate that our \newmethod{} embedding quickly adapts to incrementally-expanding search spaces, e.g. when the developers of an ML system add new algorithms. In this circumstance, existing transfer-learning PO baselines do not adapt easily, because they assume a static pipeline search space. As a remedy, we propose that when a new algorithm is added to the system after meta-training, we train only a new encoder from scratch (randomly initialized) for that new algorithm. Additionally, the meta-learned parameters for the other encoders and the aggregation layer are frozen. In this experiment, we run our method on variants of the search space when one algorithm at a time is introduced to the search space (for instance an estimator, e.g. MLP, RF, etc., is not known during meta-training, but added new to the meta-testing).

In Tables \ref{tab:omit_estimator_oboe } and \ref{tab:omit_estimator_pmf } (in Appendix), we report the results in Experiment 4 by providing the values of the average rank among five different configurations for \newmethod{}. We compare among meta-trained versions (denoted by \cmark in the column \textit{MTd.}) that omit specific estimators during meta-training (\textit{MTr.}=\cmark), or during meta-testing (\textit{MTe.}=\cmark). We also account for versions with one encoder layer denoted by \cmark in the column \textit{Enc}.

The best in all cases is the meta-learned model that did not omit the estimator (i.e. algorithm known and prior evaluations with that algorithm exist). Among the versions that omitted the estimator in the meta-training set (i.e. algorithm added new), the best configuration was the \newmethod{} which fine-tuned a new encoder for that algorithm (line \textit{Enc}=\cmark, \textit{MTd.}=\cmark, \textit{MTr.}=\cmark, \textit{MTe.}=\xmark). This version of \newmethod{} performs better than ablations with no encoder layers (i.e. only aggregation layers $\phi$), or the one omitting the algorithm during meta-testing (i.e. pipelines that do not use the new algorithm at all). The message of the results is simple: If we add a new algorithm to an ML system, instead of running PO without meta-learning (because the search space changes and existing transfer PO baselines are not applicable to the new space), we can use a meta-learned \newmethod{} and only fine-tune an encoder for a new algorithm.

\subsection{On the Inductive Bias and Meta-Learning}
\label{subsection:inductive_bias_and_pretraining}

\begin{figure} [bpt]
    \centering
    \vspace{-0.3cm}
    \includegraphics[width=.9\linewidth]{figures/reg_ablation_all.pdf}
    \caption{Comparison of the average rank for \newmethod{} with a different number of encoders under different percentages of meta-train data. The total number of layers is always the same.}
    \vspace{-0.3cm}
    \label{fig:ablation_encoders_train_size}
\end{figure}

The results of \textbf{Experiment 5} on effect of the inductive bias introduced by the encoders are presented in Figure \ref{fig:random_embeddings}. The pipelines with the same active algorithm in the estimation stage, but with different hyperparameters, lie closer in the embedding space created by a random initialized \newmethod{}, forming compact clusters characterized by the defined cluster metric (value below the plots). We formally demonstrate in Appendix \ref{appendix:theoretical_insight} that, in general, a single encoder layer is creating more compact clusters than a fully connected linear layer.

In another experiment, we assess the performance of \newmethod{} with different network sizes and meta-trained with different percentages of meta-training tasks: 0.5\%, 1\%, 5\%, 10\%, 50\%, and 100\%. As we use the Tensor-OBOE meta-dataset, this effectively means that we use 1, 3, 16, 33, 165, and 330 tasks respectively. We ran the experiment for three values of $F$. The presented scores are the average ranks among the three \newmethod{} configurations (row-wise). The average rank is computed across all the meta-test tasks and across 100 BO iterations. 

The results reported in Figure \ref{fig:ablation_encoders_train_size} indicate that deeper encoders achieve a better performance when a small number of meta-training tasks is available. In contrast, shallower encoders are needed if more meta-training tasks are available.  Apparently the deep aggregation layers $\phi$ already capture the interaction between the hyperparameter configurations across algorithms when a large meta-dataset of evaluated pipelines is given. The smaller the meta-data of evaluated pipeline configurations, the more inductive bias we need to implant in the form of per-algorithm encoders.

\subsection{Visualizing the learned embeddings}

\label{appendix:learnt_representations}
\begin{figure}
    \includegraphics[width=.99\linewidth]{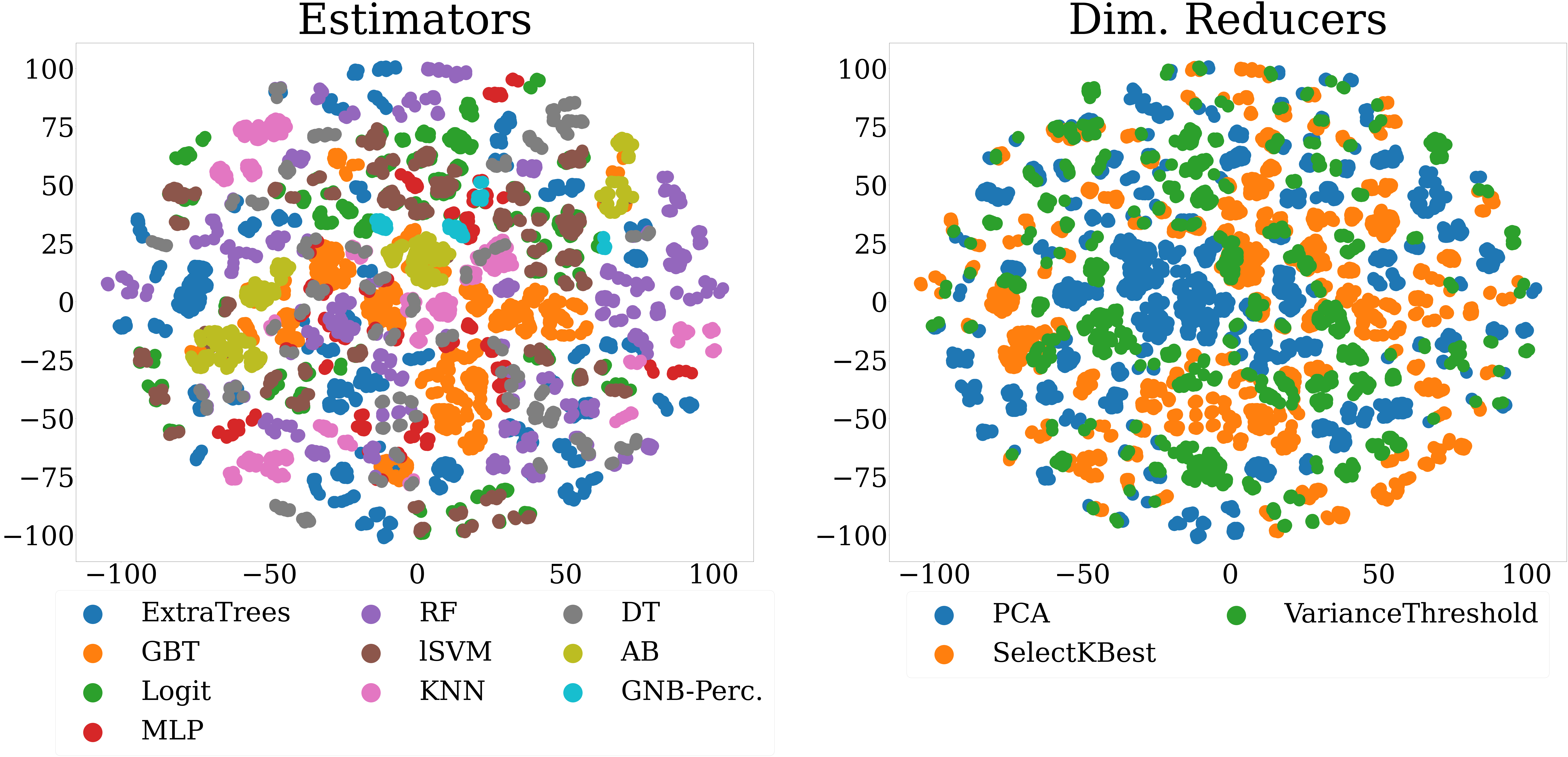}
    \caption{Pipeline embeddings produced by a meta-learned \newmethod{} using Tensor-OBOE meta-dataset. We define color markers for estimators (left) and dimensionality reducers (right).}
    \label{fig:learnt_represetations_algorithms}
\end{figure}

\begin{figure}
    \includegraphics[width=.99\linewidth]{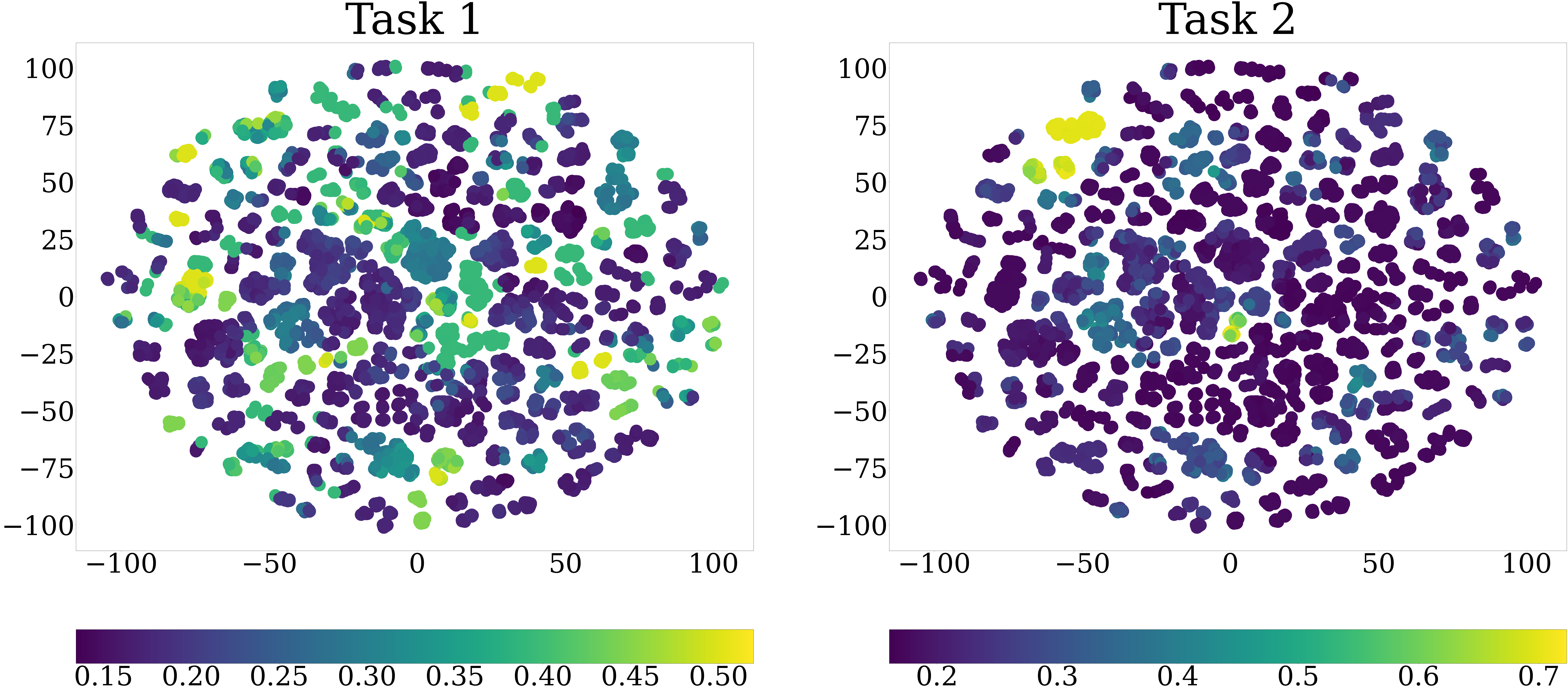}
    \caption{Pipeline embeddings produced by a meta-learned \newmethod{} using Tensor-OBOE meta-dataset. The color indicates the accuracy level of every pipeline on two different meta-testing tasks.}
    \label{fig:learnt_representations_tasks}
\end{figure}

We are interested in visualizing how the pipeline's representations cluster in the embedding space. Therefore, we train a \newmethod{} with 2-layer encoders, 2 aggregation layers, 20 output size, and $F=8$. To project the 20-dimensional embeddings into 2 dimensions, we apply TSNE (T-distributed Stochastic Neighbor Embedding). As plotted in Figure \ref{fig:learnt_represetations_algorithms}, the pipelines with the same estimator and dimensionality reducer are creating clusters. Note that embeddings of the same algorithms are forming clusters and capturing the similarity between other algorithms. The groups in this latent space are also indicators of performance on a specific task. 
In Figure \ref{fig:learnt_representations_tasks} we show the same pipeline's embeddings with a color marker indicating its accuracy on two meta-testing tasks. Top-performing pipelines (yellow color) are relatively close to each other in both tasks and build up regions of good pipelines. These groups of good pipelines are different for every task, which indicates that there is not a single pipeline that works for all tasks. Such results demonstrate how \newmethod{} maps the pipelines to an embedding space where it is easier to assess the similarity between pipelines and therefore to search for well-performing pipelines.

\section{Conclusion}

\label{section:limitations}
In this paper, we have shown that efficient Machine Learning pipeline representations can be computed with deep modular networks. Such representations help discover more accurate pipelines compared to the state-of-art approaches because they capture the interactions of the different pipelines algorithms and their hyperparameters via meta-learning and/or the architecture. Moreover, we show that introducing per-algorithm encoders helps in the case of limited meta-trained data, or when a new algorithm is added to the search space. Overall, we demonstrate that our method \newmethod{} achieves the new state-of-the-art in Pipeline Optimization. Future work could extend our representation network to model more complex use cases such as parallel pipelines or ensembles of pipelines.

\begin{acks}
This research was funded by the Deutsche Forschungsgemeinschaft (DFG, German Research Foundation) under grant number 417962828 and grant INST 39/963-1 FUGG (bwForCluster NEMO). In addition, Josif Grabocka acknowledges the support of the BrainLinks- BrainTools Center of Excellence, and the funding of the Carl Zeiss foundation through the ReScaLe project.
\end{acks}

\bibliographystyle{ACM-Reference-Format}
\bibliography{kdd_bibliography}

\appendix

\newpage
\appendix

\section{Potential Negative Societal Impacts}
\label{appendix:negative_impacts}

The meta-training is the most demanding computational step, thus it can incur in high energy consumption. Additionally, \newmethod{} does not handle fairness, so it may find pipelines that are biased by the data. 
\section{Licence Clarification}
\label{appendix:license}
The results of this work (code, data) are under license BSD-3-Clause license. Both the PMF~\cite{PMFGit}, Tensor-OBOE~\cite{OboeGit} and ZAP~\cite{ZAPGit} datasets hold the same license.

\section{Discussion on the Interactions among Components}
\label{appendix:interactions_discussion}
\begin{figure}  
    \centering
    \vspace{-0.3cm}
    \includegraphics[width=.7\linewidth]{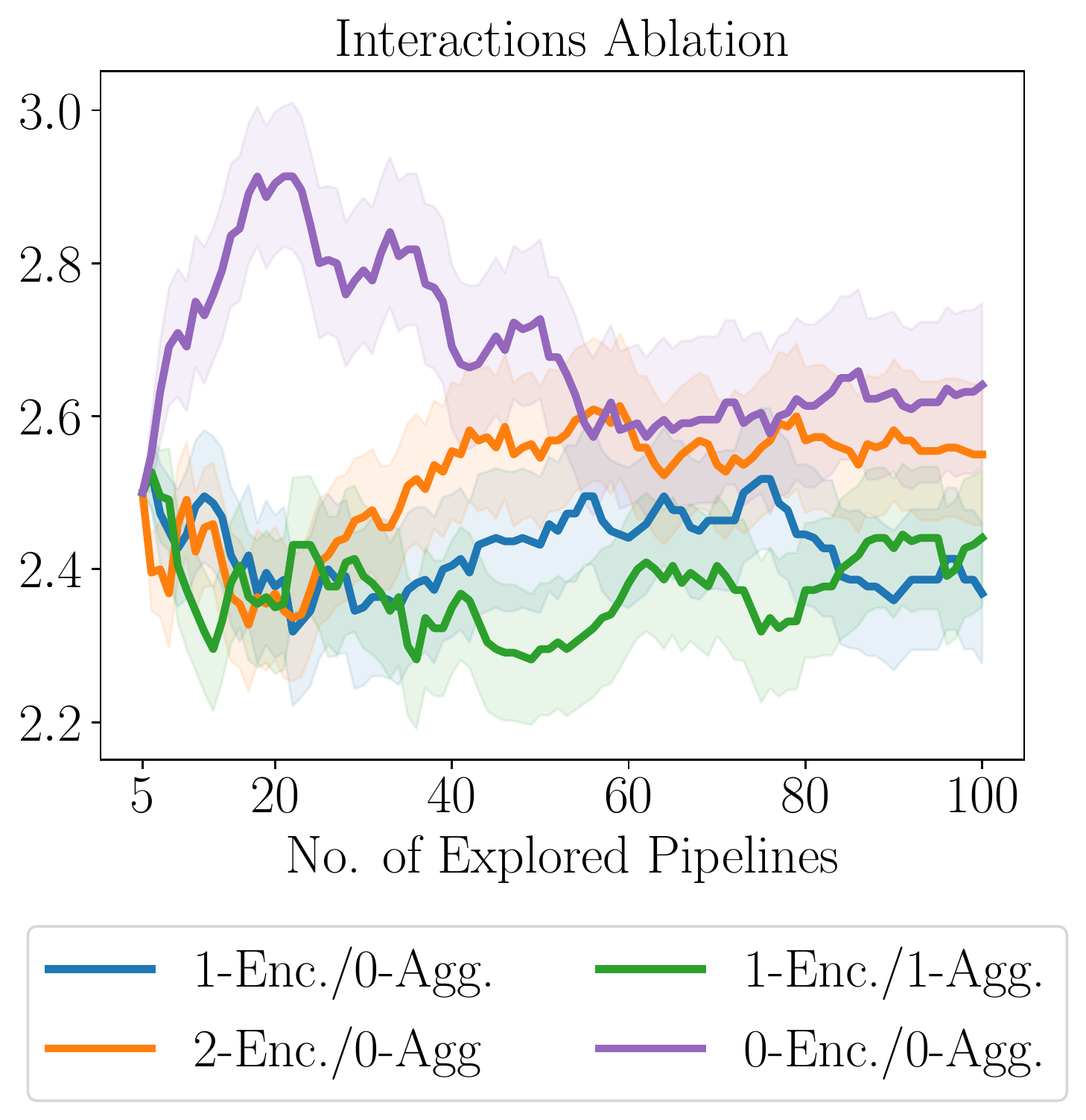}
    \caption{{Average rank for \newmethod{} with and without encoder and and aggregation layers.}}
    \vspace{-0.3cm}
    \label{fig:interactions_ablation}
\end{figure}
The encoder and aggregation layers capture interactions among the pipeline components and therefore are important to attain good performance. These interactions are reflected in the features extracted by these layers, i.e. the pipeline representations obtained by \newmethod{}. These representations lie on a metric space that captures relevant information about the pipelines and which can be used on the kernel for the Gaussian Process. Using the original input space does not allow the extraction of rich representations. To test this idea, we meta-train four versions of \newmethod{} with and without encoder and aggregation layers on our TensorOBOE meta-train set and then test on the meta-test split. In Figure \ref{fig:interactions_ablation}, we show that the best version is obtained when using both encoder (\textit{Enc.}) and aggregation (\textit{Agg.}) layers (green line), whereas the worst version is obtained when using the original input space, i.e. no encoder and no aggregation layers. Having an encoder helps more than otherwise, thus it is important to capture interactions among hyperparameters in the same stage. As having an aggregation layer is better than not, capturing interactions among components from different stages is important.

\section{Architectural Implementation}

\begin{figure}

\centering
 \includegraphics[width=0.99\linewidth]{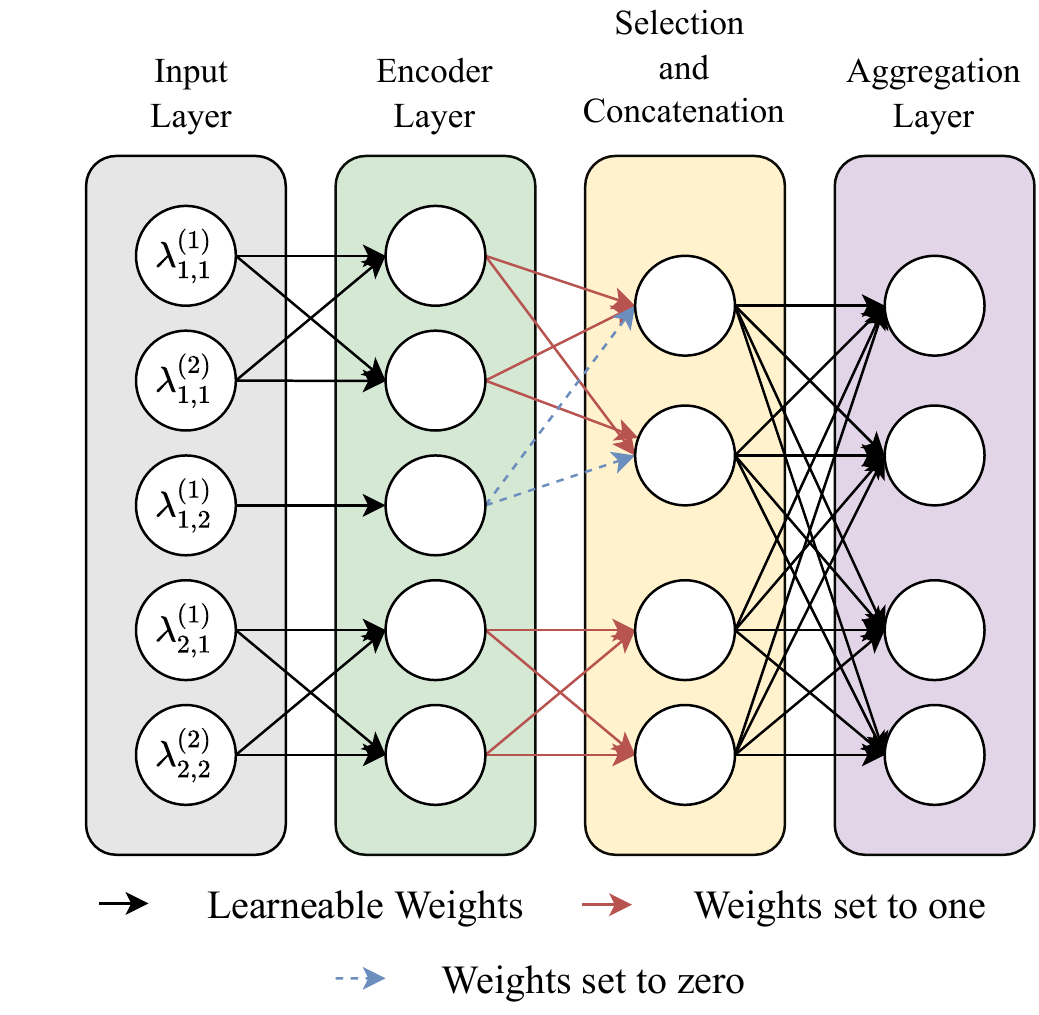}  
\caption{Example of the Implementation of \newmethod{} as MLP. $\lambda^{(k)}_{i,j}$ indicates the $k$-th hyperparameter of the $j$-th algorithm in the $i$-th stage. In this architecture, the first stage has two algorithms, thus two encoders. The algorithm $1$ is active for stage $1$. The second stage has only one algorithm.}
\label{fig:architecture_implementation}

\end{figure}

\newmethod{}'s architecture (encoder layers + aggregated layers) can be formulated as a Multilayer Perceptron (MLP) comprising three parts (Figure \ref{fig:architecture_implementation}). The first part of the network that builds the layers with encoders is implemented as a layer with masked weights. We connect the input values corresponding to the hyperparameters $ \lambda_{(i,j)}$ of the $j$-th algorithm of the $i$-th stage to a fraction of the neurons in the following layer, which builds the encoder. The rest of the connections are dropped. The second part is a layer that selects the output of the encoders associated with the active algorithms (one per stage) and concatenates their outputs (\textit{Selection \& Concatenation}). The layer's connections are fixed to be either one or zero during forward and backward passes. Specifically, they are one if they connect outputs of active algorithms' encoders, and zero otherwise. The last part, an \textit{aggregation layer}, is a fully connected layer that learns interactions between the concatenated output of the encoders. By implementing the architecture as an MLP instead of a multiplexed list of components(e.g. with a module list in PyTorch), faster forward and backward passes are obtained. We only need to specify the selected algorithms in the forward pass so that the weights in the \textit{Encoder Layer} are masked and the ones in the \textit{Selection \& Concatenation} are accordingly set. After this implementation, notice that \newmethod{} can be interpreted as an MLP with \textbf{sparse} connections. Further details on the architecture are discussed in Appendix \ref{appendix:architecture_details}.

\section{Additional Results}
\label{appendix:additional_results}

In this section, we present further results. Firstly, we show an ablation of the factor that determines the number of hidden units ($F$) in Figure \ref{fig:rank_ablation}. It shows that $F=8$ attains the best performance after exploring 100 pipelines in both datasets. Additionally, we present the average regret for the ablation of $F$, and the results of Experiment 1 and 2 in Figures \ref{fig:regret_ablation}, \ref{fig:regret_non_transfer} and \ref{fig:regret_transfer} respectively. The average regret is defined as $y_{max}-y^*$, where $y_max$ is the maximum accuracy possible within the task and $y^*$ is the maximum observed accuracy.

\begin{figure}[ht]
    \includegraphics[width=.99\linewidth]{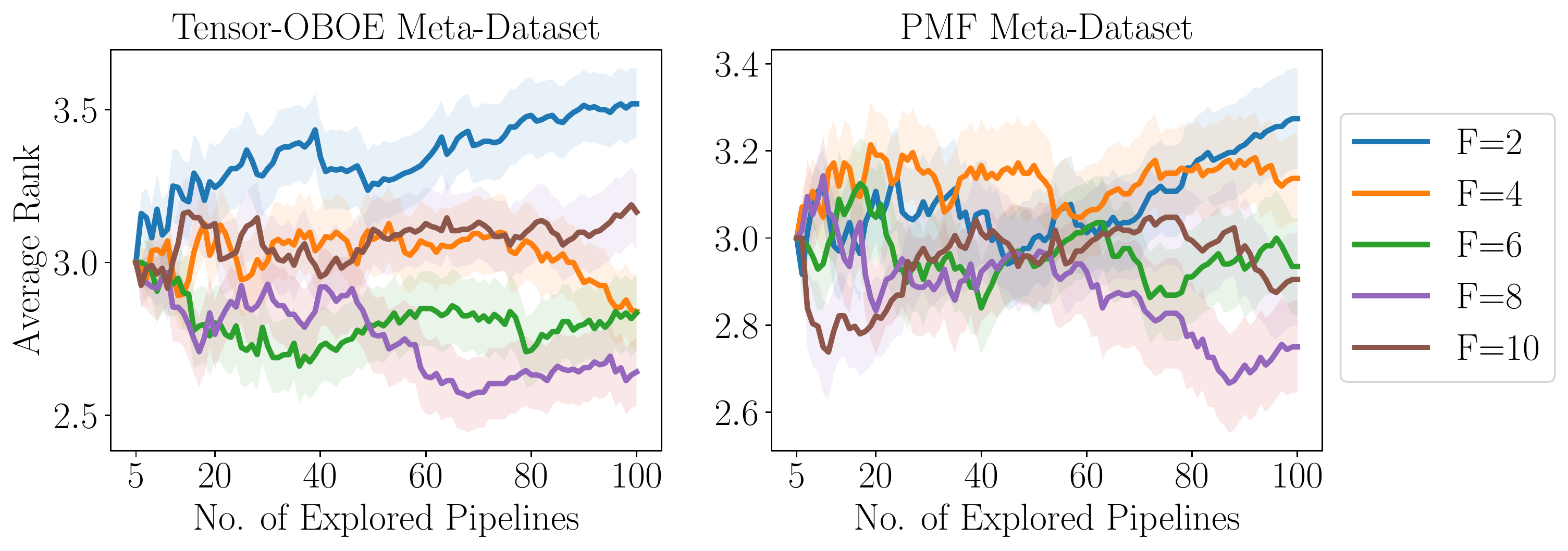}
    \caption{Comparison of different $F$ values in \newmethod{} (Rank).}
    \label{fig:rank_ablation}
\end{figure}

\begin{figure}[ht]
    \includegraphics[width=.99\linewidth]{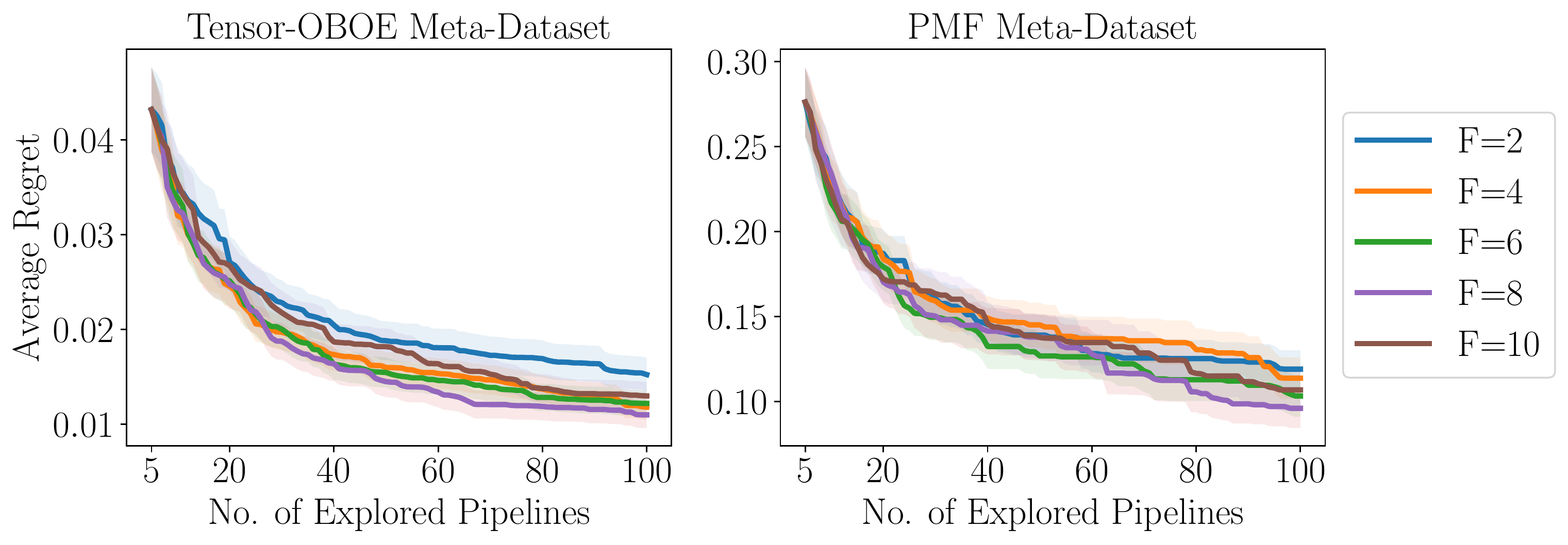}
    \caption{Comparison of different $F$ values in \newmethod{} (Regret).}
    \label{fig:regret_ablation}
\end{figure}

\begin{figure}[ht]
    \includegraphics[width=.99\linewidth]{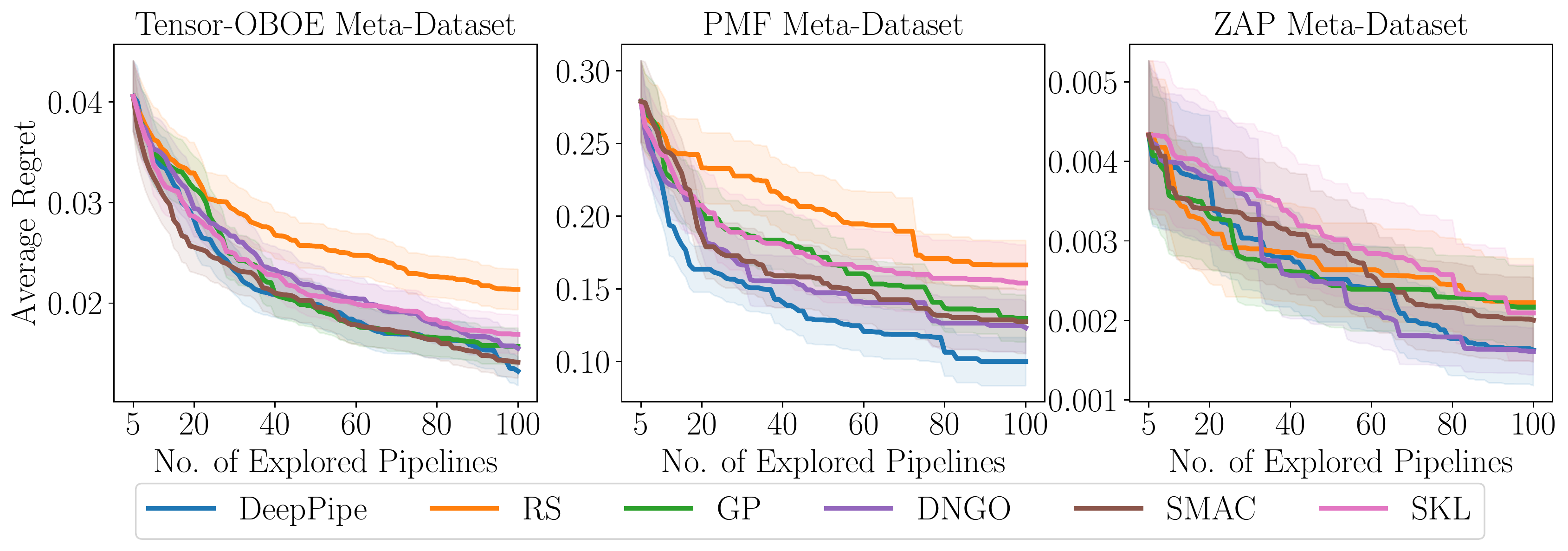}
    \caption{Comparison of \newmethod{} vs. non transfer-learning PO methods in Experiment 1 (Regret)}
    \label{fig:regret_non_transfer}
\end{figure}

\begin{figure}[ht]
    \includegraphics[width=1.0\linewidth]{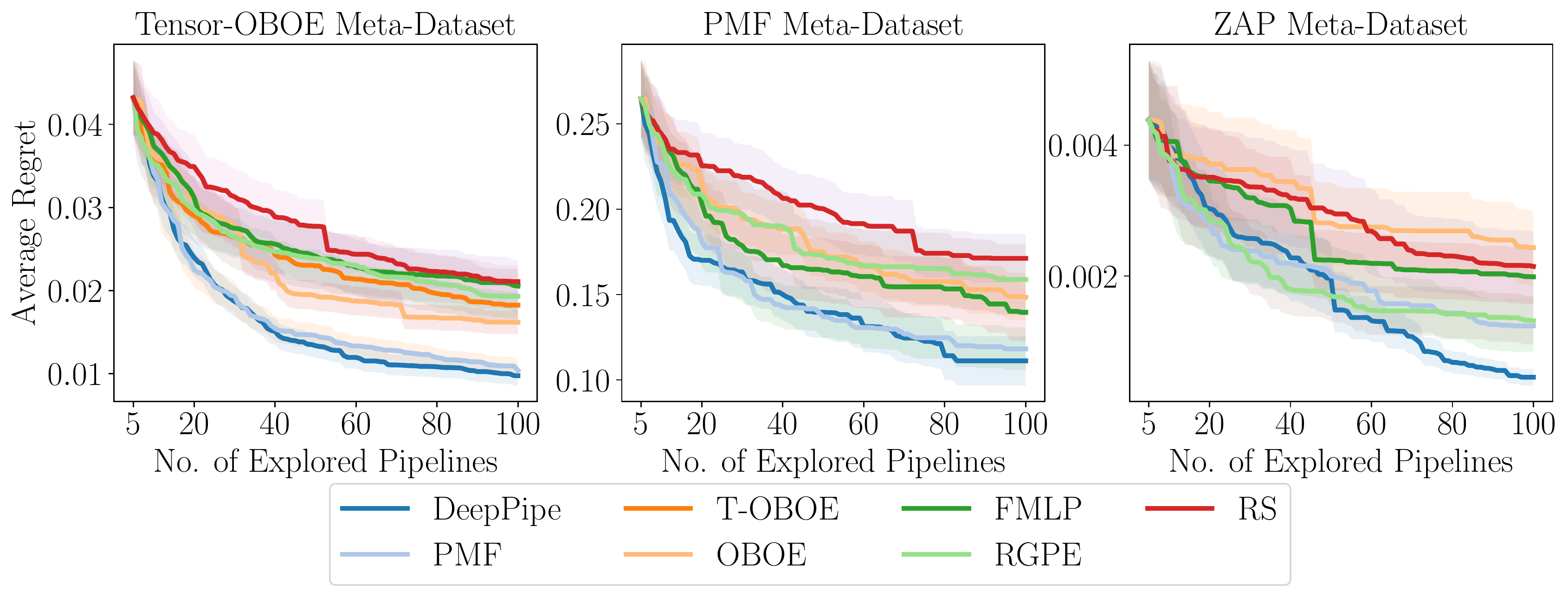}
    \caption{Comparison of Regret in \newmethod{} vs. transfer-learning PO methods in Experiment 2 (Regret)}
    \label{fig:regret_transfer}
\end{figure}


Table \ref{tab:omit_estimator_pmf } presents the extended results of omitting estimators in the PMF Dataset. From these, we draw the same conclusion as in the same paper: having encoders help to obtain better performance when a new algorithm is added to a pipeline.

\begin{table*}[t]
\caption{Average rank among \newmethod{} variants for newly-added algorithms (PMF)}\label{tab:omit_estimator_pmf }
\begin{tabular}{ccccccccccccc}\toprule
\multirow{2}{*}{Enc.} &\multirow{2}{*}{MTd.} &\multicolumn{2}{c}{Omitted in} &\multicolumn{8}{c}{Omitted Estimator} \\ \cmidrule(ll){3-4} \cmidrule(ll){5-12}
& &MTr. &MTe. &ET &RF &XGBT &KNN &GB &DT &Q/LDA &NB \\\midrule
\cmark & \cmark & \cmark & \cmark &3.1527 &3.1645 &3.2109 &3.2541 &3.2874 &3.2741 &3.1911 &3.0263 \\
\cmark & \xmark & \cmark & \xmark &3.2462 &3.3208 &3.2592 &3.3180 &3.2376 &3.2249 &3.3557 &3.3993 \\
\cmark & \cmark & \xmark & \xmark & \textbf{2.5710} & \textbf{2.5996} &\textbf{2.4011} &\textbf{2.5947} &\textbf{2.6301} &\textbf{2.5664} &\textbf{2.6252} &\textbf{2.6214} \\
\xmark &\cmark &\cmark & \xmark&3.0464 & \underline{2.8550} &3.0850 & \underline{2.8845}  &2.9397  &3.0316 &2.9530 &3.0596 \\
\cmark &\cmark &\cmark & \xmark&\underline{2.9838} &3.0601 & \underline{3.0439} & 2.9486 & \underline{2.9051} & \underline{2.9029} &\underline{2.8750} &\underline{2.8934} \\
\bottomrule
\end{tabular}
\end{table*}

We carry out an ablation to understand the difference between the versions of Deep Pipe with/without encoder and with/without transfer-learning using ZAP Meta-dataset. As shown in Figure \ref{fig:zap_ablation_on_rank}, the version with transfer learning and one encoder performs the best, thus, highlighting the importance of encoders in transfer learning our \newmethod{} surrogate.

\begin{figure}[!b]
     \centering
    \includegraphics[width=.99\linewidth]{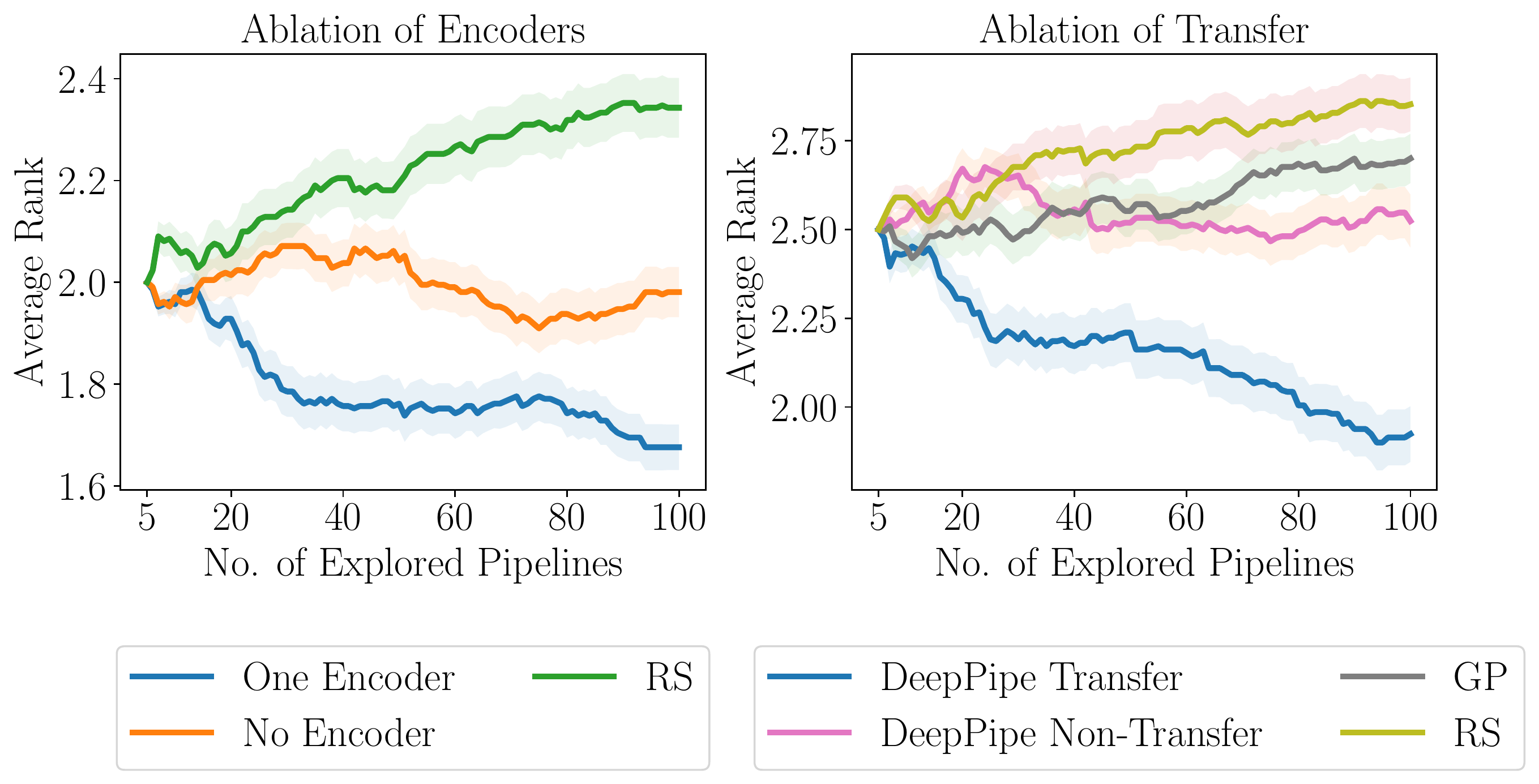}
    \caption{Ablations on the ZAP meta-dataset}
    \label{fig:zap_ablation_on_rank}
\end{figure}

\section{Architecture Details}
\label{appendix:architecture_details}

The input to the kernel has a dimensionality of $Z$=20. We fix it, to be the same as the output dimension for PMFs. The number of neurons per layer, as mentioned in the main paper, depends on $F$. Consider an architecture with with no encoder layers and $\ell_a$ aggregation layers, and hyperparameters $\Lambda_{i,j}, i \in \{1 \dots N\}, j \in \{1 \dots M_i \}$ (following the notation in section \ref{section:pipeline_embedding}) with $L_i = \max_j |\Lambda_{i,j}|$, then the number of weights (omitting biases for the sake of simplicity) will be:

\begin{equation}
    \left(\sum_{i,j}  |\Lambda_{i,j}| \right) \cdot  \left( F \cdot  \sum_i L_i  \right) + (\ell_a-1) \left( F \cdot \sum_i L_i \right)^2
\end{equation}

If the architecture has $\ell_e$ encoder layers and $\ell_a$ aggregation layers, then the number of weights is given by:

\begin{equation}
    \sum_{i,j} |\Lambda_{i,j}|\cdot  \left( F \cdot L_i \right) +  (\ell_e-1)\sum_i M_i  \cdot \left( F \cdot L_i \right)^2 + \ell_a \left( F \cdot \sum_i L_i \right)^2
\end{equation}

 In other words, the aggregation layers have $F\cdot \sum_i L_i$ hidden neurons, whereas every encoder from the $i$-th stage has $F \cdot  l_i$ neurons per layer. The input sizes are $\sum_{i,j}  |\Lambda_{i,j}|$ and $|\Lambda_{i,j}|$ for both cases respectively. The specific values for $|\Lambda_{i,j}|$ and $ L_i$ per search space are specified in Appendix \ref{appendix:search_spaces}.

In the search space for PMF, we group the algorithms related to Naive Bayers (MultinomialNB, BernoulliNB, GaussianNB) in a single encoder. In this search space, we  also group LDA and QDA. In the search space of TensorOboe, we group GaussianNB and Perceptron as they do not have hyperparameters. Given these considerations, we can compute the input size and the weights per search space as function of $\ell_a, \ell_e, F$ as follows:

(i) Input size:

\begin{equation}
\begin{split}
    \text{\# Input size (PMF) } =  \sum_{i,j}  |\Lambda_{i,j}| = 72
    \\
    \text{\# Input (TensorOboe) } =  \sum_{i,j}  |\Lambda_{i,j}| = 37
    \\
     \text{\# Input (ZAP) } =  \sum_{i,j}  |\Lambda_{i,j}| = 35
\end{split}
\end{equation}

(ii) Number of weights for architecture without encoder layers:
\begin{equation}
\begin{split}
    \text{\# Weights (PMF) } = 720 \cdot F +  256 \cdot (\ell_a-1)\cdot F^2 
    \\
    \text{\# Weights (TensorOboe) } = 444 \cdot F +  144 \cdot (\ell_a-1)\cdot F^2
    \\
    \text{\# Weights (ZAP) } = 1085 \cdot F +   961  \cdot (\ell_a-1)\cdot F^2
\end{split}
\end{equation}

(iii) Number of weights for architecture with encoder layers:

\begin{equation}
\begin{split}
    \text{\#  Weights (PMF) } = 886 \cdot F +  \left(1376 \cdot (\ell_e-1) + 256 \cdot \ell_a \right)\cdot F^2 
    \\
    \text{\# Weights (TensorOboe) } = 161 \cdot F +  \left(271 \cdot (\ell_e-1) + 144 \cdot \ell_a \right)\cdot F^2
    \\
    \text{\# Weights (ZAP) } = 35 \cdot F +  \left( 965 \cdot (\ell_e-1) + 961 \cdot \ell_a \right) \cdot F^2    
\end{split}
\end{equation}

\begin{figure}
    \centering
    \includegraphics[width=0.99\linewidth]{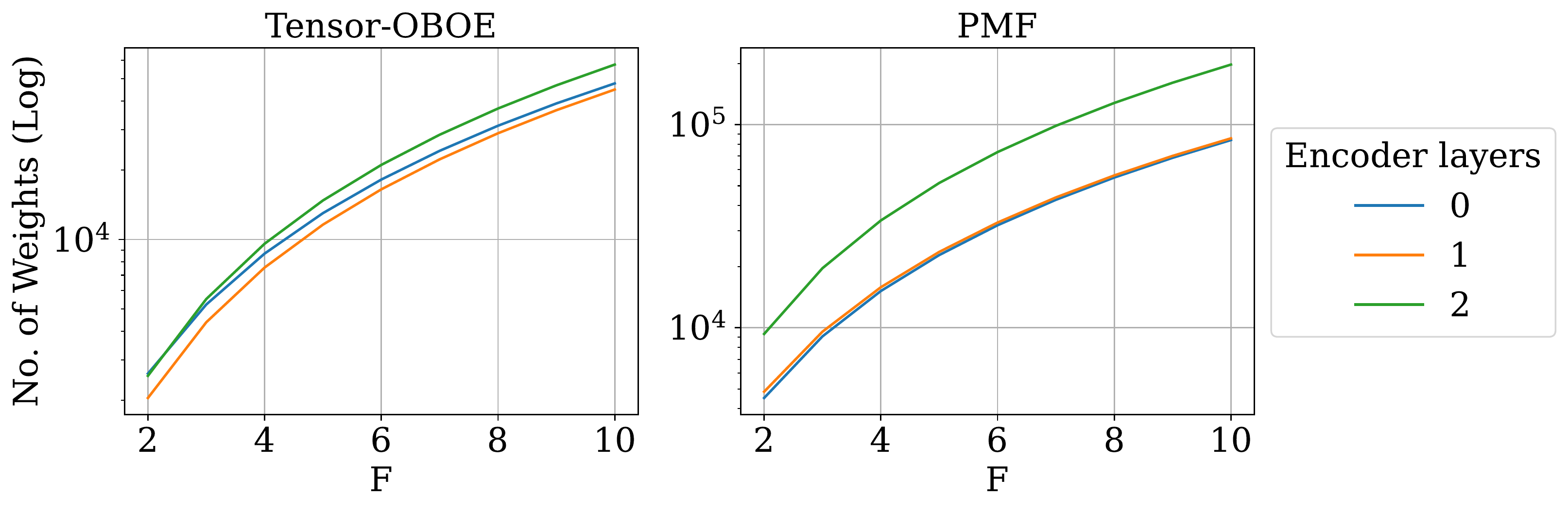}
    \caption{Number of weights in the MLP for a given value of $F$ and encoder layers.}
    \label{fig:weights_per_f}
\end{figure}

According the previous formulations, Figure \ref{fig:weights_per_f} shows how many parameters (only weights) the MLP has given a specific value of F and of encoder layers. We fix the total number of layers to four. Notice that the difference in the number of parameters between an architecture with 1 and 2 encoder layers is small in both search spaces. 
Notice that we associate algorithms with no hyperparameters to the same encoder in our experiments (Appendix \ref{appendix:search_spaces}). Moreover, we found that adding the One-Hot-Encoding of the selected algorithms per stage as an additional input is helpful. Therefore, the input dimensionality of the aggregated layers is equal to the dimension after concatenating the encoder's output $F \cdot  \sum_i (Q_i+M_i )$.

\section{Abbreviations}
\label{appendix:abbreviations}
(i) Abbreviations in Table 2:

1) ET: ExtraTrees, 2) GBT: Gradient Boosting, 3) Logit: Logistict Regression 4) MLP: Multilayer Perceptron 5) RF: Random Forest, 6) lSVM: Linear Support Vector Machine,  7) kNN: k Nearest Neighbours, 8) DT: Decision Trees, 9) AB: AdaBoost, 10) GB/PE= Gaussian Naive Bayes/Perceptron.

(ii) Abbreviations in Table 3:

1) ET: ExtraTrees, 2) RF: Random Forest , 3) XGBT: Extreme Gradient Boosting, 4) kNN: K-Nearest Neighbours, 5) GB: Gradient Boosting, 6) DT: Decision Trees, 7) Q/LDA: Quadratic Discriminant Analysis/ Linear Discriminant Analysis, 8) NB: Naive Bayes.

\section{Theoretical Insight of Hypothesis 5}

Here, we formally demonstrate that the \newmethod{} with encoder layers is grouping hyperparameters from the same algorithm in the latent space, better than \newmethod{} without encoders, formulated on Corollary \ref{corollary:assumption}, which is supported by Proposition \ref{proposition:encoders_per_stage}.

\label{appendix:theoretical_insight}
\begin{lemma}
\label{lemma:expected_value_of_norm}
Given $\boldsymbol{w} \in \mathrm{R}^{M}$, a vector of weights with independent and identically distributed components $w_i \in \{w_1,...,w_M\}$ such that $w_i \sim p(w)$, the expected value of the square of the norm $\mathbb{E}_{p(w)}(||\boldsymbol{w}||^2)$ is given by $M\cdot(\mu_{w}^2+\sigma_{w}^2)$, where $\mu_{w}$ and $\sigma_{w}$ are the mean and standard deviation of $p(w)$ respectively.
\end{lemma}

\begin{proof}
\begin{align}
\centering
\mathbb{E}_{p(w)}\left(||\boldsymbol{w}||^2 \right) &= \mathbb{E}_{p(w)}\left(\sum_{i=1}^M w_i^2\right) \\
& = \sum_{i=1}^M  \mathbb{E}_{p(w)} (w_i^ 2) \\
& =\sum_{i=1}^M  \mu_{w}^2+\sigma_{w}^2 \\
& = M\cdot(\mu_{w}^2+\sigma_{w}^2)
\end{align}
\end{proof}

\begin{lemma}
\label{lemma:expected_value_linear_function_ouput}
Consider a linear function with scalar output $z=\boldsymbol{w}^T\boldsymbol{x}$ where $\boldsymbol{w} \in \mathrm{R}^{M \times 1}$ is the vector of weights with components $w_i, i \in \{1,...,M\}$, $\boldsymbol{x} \in \mathrm{R}^{M \times 1}$ are the input features. Moreover, consider the weights are independently and identically distributed $w_i \sim p(w)$. The expected value of the norm of the output is given by $\mathbb{E}_{p(w)} \left( || \boldsymbol{w}^T\boldsymbol{x} ||^2 \right)= (\mu_{w}^2+\sigma_{w}^2) \cdot  ||\boldsymbol{x}||^2 + \mu_w^2 \cdot \sum^M_{i=1} \sum^{i-1}_{j=1}  x_i \cdot x_j$.
\end{lemma}

\begin{proof}
\begin{align}
    & \mathbb{E}_{p(w)} \left( ( \boldsymbol{w}^T\boldsymbol{x} )^2 \right) \\ & = \mathbb{E}_{p(w)} \left( \sum_{i=1}^M w_i\cdot x_i \right)^ 2 \\
    & = \mathbb{E}_{p(w)} \left( \sum_{i=1}^M (w_i\cdot x_i )^ 2  + \sum^M_{i=1} \sum^{i-1}_{j=1} w_i \cdot w_j \cdot x_i \cdot x_j   \right) \\
    & =   \sum_{i=1}^M \mathbb{E}_{p(w)}(w_i^ 2)\cdot x_i^ 2  + 2\cdot\sum^M_{i=1} \sum^{i-1}_{j=1} \mathbb{E}_{p(w)}(w_i \cdot w_j) \cdot x_i \cdot x_j   \\
\end{align}
Since $w_i, w_j$ are independent  then $\mathbb{E}_{p(w)}(w_i \cdot w_j)=\mathbb{E}_{p(w)}(w_i) \cdot \mathbb{E}_{p(w)}(w_j)=\mu_w^2$. Moreover, with a slight abuse in notation, we denote $\sum^M_{i=1} \sum^{i-1}_{j=1}  x_i \cdot x_j = \boldsymbol{x} \otimes  \boldsymbol{x}$. Given lemma \ref{lemma:expected_value_of_norm}, we obtain:
\begin{align}
    \mathbb{E}_{p(w)} \left( ( \boldsymbol{w}^T\boldsymbol{x} )^2 \right)  = (\mu_w^2 + \sigma_w^2) \cdot ||\boldsymbol{x}||^2 + 2\cdot \mu_w^2 \cdot \boldsymbol{x} \otimes  \boldsymbol{x} = D_w (\boldsymbol{x}) \\
\end{align}
where $D_w(\cdot)$ is introduced as an operation to simplify the notation.
\end{proof}

\begin{proposition}
\label{proposition:encoders_per_stage}

Consider two vectors $ \boldsymbol{x}^{\prime}, \hat{\boldsymbol{x}} \in \mathrm{R}^{M}$, and two weight vectors $\hat{\boldsymbol{w}}$ and $\boldsymbol{w}^{\prime}$, $\hat{\boldsymbol{w}}^ T\hat{\boldsymbol{x}} \in \mathrm{R}, {\boldsymbol{w}^{\prime}}^T\boldsymbol{x}^{\prime}\in \mathrm{R}$, such that the weights are iid. Then $\mathbb{E}_{p(w)} \left( ( \hat{\boldsymbol{w}}^ T\hat{\boldsymbol{x}} -{\boldsymbol{w}^{\prime}}^T\boldsymbol{x}^{\prime})^ 2 \right) > \mathbb{E}_{p(w)}\left( ( \hat{\boldsymbol{w}}^ T\hat{\boldsymbol{x}} -\hat{\boldsymbol{w}}^T\boldsymbol{x}^{\prime})^ 2 \right)$.
\end{proposition}
\begin{proof}
Using lemma \ref{lemma:expected_value_linear_function_ouput} and decomposition the argument within square:
\begin{align}
   & \mathbb{E}_{p(w)}(( \hat{\boldsymbol{w}}^ T\hat{\boldsymbol{x}} - {\boldsymbol{w}^{\prime}}^T\boldsymbol{x}^{\prime})^ 2) \\
& =  \mathbb{E}_{p(w)} \left( (\hat{\boldsymbol{w}}^ T\hat{\boldsymbol{x}})^2 + ({\boldsymbol{w}^{\prime}}^T\boldsymbol{x}^{\prime})^2 -2 \cdot \hat{\boldsymbol{w}}^ T\hat{\boldsymbol{x}} \cdot  {\boldsymbol{w}^{\prime}}^T\boldsymbol{x}^{\prime} \right) \\
& = D_w(\hat{\boldsymbol{x}}) + D_w(\boldsymbol{x}^{\prime}) - 2 \cdot \mathbb{E}_{p(w)} (\hat{\boldsymbol{w}}^ T\hat{\boldsymbol{x}} \cdot  {\boldsymbol{w}^{\prime}}^T\boldsymbol{x}^{\prime})\\
    & = D_w(\hat{\boldsymbol{x}}) + D_w(\boldsymbol{x}^{\prime}) - 2 \cdot \mathbb{E}_{p(w)} ( \sum_{i=1}^{\hat{M}} \hat{w}_i\cdot \hat{x}_i \sum_{j=1}^{M^{\prime}} {w_j}^{\prime}\cdot {x_j}^{\prime}) \\
    & = D_w(\hat{\boldsymbol{x}}) + D_w(\boldsymbol{x}^{\prime}) - 2 \cdot \mathbb{E}_{p(w)} ( \sum_{i=1}^{\hat{M}} \sum_{j=1}^{M^{\prime}} {w_j}^{\prime}\cdot {x_j}^{\prime} \cdot \hat{w}_i\cdot \hat{x}_i ) \\
    & = D_w(\hat{\boldsymbol{x}}) + D_w(\boldsymbol{x}^{\prime}) - 2 \cdot  \sum_{i=1}^{\hat{M}} \sum_{j=1}^{M^{\prime}} \mathbb{E}_{p(w)} ({w_j}^{\prime}\cdot \hat{w}_i) \cdot {x_j}^{\prime}\cdot \hat{x}_i 
\end{align}

Since $\hat{\boldsymbol{w}}$ and $\boldsymbol{w}^{\prime}$ are independent, then $\mathbb{E}_{p(w)} ({w_j}^{\prime}\cdot \hat{w}_i) = \mathbb{E}_{p(w)}({w_j}^{\prime}) \cdot  \mathbb{E}_{p(w)}(\hat{w}_i) = \mu_w^ 2$. Thus,

\begin{align}
    \mathbb{E}_{p(w)} \left(( \hat{\boldsymbol{w}}^ T\hat{\boldsymbol{x}} - {\boldsymbol{w}^{\prime}}^T\boldsymbol{x}^{\prime})^ 2 \right)    & = D_w(\hat{\boldsymbol{x}}) + D_w(\boldsymbol{x}^{\prime}) - 2 \cdot \textcolor{red}{\mu_w^ 2} \cdot  \sum_{i=1}^{\hat{M}} \sum_{j=1}^{M^{\prime}} {x_j}^{\prime}\cdot \hat{x}_i 
\end{align}

When computing  $\mathbb{E}_{p(w)}\left(( \hat{\boldsymbol{w}}^ T\hat{\boldsymbol{x}} -\hat{\boldsymbol{w}}^T\boldsymbol{x}^{\prime})^ 2 \right)$, we see that the weights are not independent, thus $\mathbb{E}_{p(w)} (\hat{w}_i \cdot \hat{w}_i) =\mu_w^2+\sigma_w^2 $, and

\begin{align}
   & \mathbb{E}_{p(w)} \left(( \hat{\boldsymbol{w}}^ T\hat{\boldsymbol{x}} - \hat{\boldsymbol{w}}^T\boldsymbol{x}^{\prime})^ 2 \right)  \\   & = D_w(\hat{\boldsymbol{x}}) + D_w(\boldsymbol{x}^{{\prime}}) - 2 \cdot \textcolor{red} {(\mu_w^ 2 + \sigma_w^2)} \cdot  \sum_{i=1}^{\hat{M}} \sum_{j=1}^{M^{\prime}} {x_j^{\prime}}\cdot \hat{x}_i \\
    & <  D_w(\hat{\boldsymbol{x}}) + D_w(\boldsymbol{x}^{\prime}) - 2 \cdot \textcolor{red}{\mu_w^ 2} \cdot  \sum_{i=1}^{\hat{M}} \sum_{j=1}^{M^{\prime}} {x_j}^{\prime}\cdot \hat{x}_i  \\ 
    & <   \mathbb{E}_{p(w)} \left(( \hat{\boldsymbol{w}}^ T\hat{\boldsymbol{x}} - {\boldsymbol{w}^{\prime}}^T\boldsymbol{x}^{\prime})^ 2 \right)  
\end{align}
\end{proof}

\begin{corollary}
 \label{corollary:assumption}
 
A random initialized \newmethod{} with encoder layers induces an assumption that two hyperparameter configurations of an algorithm should have more similar performance than hyperparameter configurations from different algorithms.
\end{corollary}

\begin{proof}
 Given two hyperparameter configurations $\lambda^{(l)}, \lambda^{(m)}$ from an algorithm, and a third hyperparameter configuration $\lambda^{(n)}$ from a different algorithm, every random initialized encoder layer from \newmethod{} maps the hyperparameters $\lambda^{(l)}, \lambda^{(m)}$ to latent dimensions $z^{(l)},z^{(m)}$ that are closer to each other than to $z^{(n)}$, i.e. the expected distance among the output of the encoder layer will be $\mathbb{E}_{p(w)}(||z^{l}-z^{m}||) < \mathbb{E}_{p(w)}(||z^{l}-z^{n}||)$ based on Proposition \ref{proposition:encoders_per_stage}.  Since \newmethod{} uses a kernel such that $\kappa(\boldsymbol{x}, \boldsymbol{x}^{\prime}) = \kappa(\boldsymbol{x}- \boldsymbol{x}^{\prime})$, their similarity will increase, when the distance between two configurations decreases. Thus, according to Equation \ref{equation:posterior}, they will have correlated performance. 
\end{proof}

\section{Meta-Dataset Search Spaces}
\label{appendix:search_spaces}

We detail the search space composition in Tables \ref{tab:search_space_pmf } (PMF), \ref{tab:search_space_oboe} (TensorOBOE) and \ref{tab:search_space_zap} (ZAP). We specify the stages, algorithms, hyperparameters, number of components per stage $M_i$, the number of hyperparameters per algorithm $|\lambda_{i,j}|$, and the maximum number of hyperparameters found in an algorithm per stage $Q_i$.  For the ZAP meta-dataset, we defined a pipeline with two stages: (i) \textit{Architecture}, which specifies the type or architecture used (i.e. ResNet18, EfficientNet-B0, EfficientNet-B1, EfficientNet-B2), and (ii) \textit{Optimization-related Hyperparameters} that are shared by all the architectures.


\begin{table*}[!htp]\centering
\caption{Search Space for PMF Meta-Dataset}\label{tab:search_space_pmf }
\scriptsize
\begin{tabular}{lllllll}\toprule
\textbf{Stage} &\textbf{$L_i$} &\textbf{$M_i$} &\textbf{Algorithm} &\textbf{$|\Lambda_{i,j}|$} &\textbf{Hyperparameters} \\\midrule
\multirow{2}{*}{Preprocessor} &\multirow{2}{*}{3} &\multirow{2}{*}{2} &Polynomial &3 &include\_bias, interaction\_only, degree \\ \cmidrule{4-6}
& & &PCA &2 &keep\_variance, whiten \\ \hline
\multirow{11}{*}{Estimator} &\multirow{11}{*}{13} &\multirow{11}{*}{8} &ExtraTrees &9 & \makecell[l]{bootstrap, min\_samples\_leaf, n\_estimators, max\_features,\\  min\_weight\_fraction\_leaf, min\_samples\_split, max\_depth } \\ \cmidrule{4-6}
& & &RandomForest &10 &\makecell[l]{bootstrap, min\_samples\_leaf, n\_estimators, max\_features, \\ min\_weight\_fraction\_leaf, min\_samples\_split, max\_depth, \\ criterion\_entropy, criterion\_gini }\\ \cmidrule{4-6}
& & &XgradientBoosting &13 & \makecell[l]{reg\_alpha, col\_sample\_bytree, colsample\_bylevel, scale\_pos\_weight, \\  learning\_rate,\\ max\_delta\_step, base\_score, n\_estimators, subsample,\\ reg\_lambda, min\_child\_weight, max\_depth,  gamma } \\ \cmidrule{4-6}
& & &kNN &4 &p, n\_neighbors, weights\_distance, weights\_uniform \\ \cmidrule{4-6}
& & &GradientBoosting &10 & \makecell[l]{max\_leaf\_nodes, learning\_rate, min\_samples\_leaf,\\ n\_estimators, subsample, min\_weight\_fraction\_leaf, max\_features, \\ min\_samples\_split, max\_depth, loss\_deviance } \\ \cmidrule{4-6}
& & &DecisionTree &9 & \makecell[l]{max\_leaf\_nodes, min\_samples\_leaf, max\_features, \\ min\_weight\_fraction\_leaf, min\_samples\_split, max\_depth, \\splitter\_best, criterion\_entropy, criterion\_gini }\\ \cmidrule{4-6}
& & &LDA &6 & \makecell[l]{shrinkage\_factor, n\_components, tol, shrinkage\_-1, \\ shrinkage\_auto, shrinkage\_manual }\\ \cmidrule{4-6}
& & &QDA &1 & reg\_param \\ \cmidrule{4-6}
& & &BernoulliNB &2 &alpha, fit\_prior \\ \cmidrule{4-6}
& & &MultinomialNB &2 &alpha, fit\_prior \\ \cmidrule{4-6}
& & &GaussianNB &1 &apply\_gaussian\_nb \\
\bottomrule
\end{tabular}
\end{table*}

\begin{table*}[!htp]\centering
\caption{Search Space for Tensor-OBOE Meta-Dataset}\label{tab:search_space_oboe}
\scriptsize
\begin{tabular}{lllllll}\toprule
\textbf{Stage} &\textbf{$L_i$} &\textbf{$M_i$} &\textbf{Algorithm} &\textbf{$|\Lambda_{i,j}|$} &\textbf{Hyperparameters} \\\midrule
Imputer &4 &1 &SimpleImputer &4 & \makecell[l]{Strategy\_constant, Strategy\_mean,Strategy\_median, \\Strategy\_most\_frequent }\\ \midrule
Encoder &1 &1 &OneHotEncoder &1 &Handle\_unknown\_ignore \\ \midrule
Scaler &1 &1 &StandardScaler &1 &- \\ \midrule
\multirow{3}{*}{Dim. Reducer} &\multirow{3}{*}{1} &\multirow{3}{*}{3} &PCA &1 &N\_components \\ \cmidrule{4-6}
& & &SelectKBest &1 &K \\ \cmidrule{4-6}
& & &VarianceThreshold &1 &- \\ \midrule
\multirow{11}{*}{Estimator} &\multirow{11}{*}{5} &\multirow{11}{*}{10} &ExtraTrees &3 &min\_samples\_split, criterion\_entropy, criterion\_gini \\ \cmidrule{4-6}
& & &Gradient Boosting &4 & \makecell[l]{learning\_rate, max\_depth, max\_features\_None,\\ max\_features\_log2 } \\ \cmidrule{4-6}
& & &Logit &5 &C, penalty\_l1, penalty\_l2, sovler\_liblinear, solver\_saga \\ \cmidrule{4-6}
& & &MLP &5 & \makecell[l]{alpha, learning\_rate\_init, learning\_rate\_adaptive, \\ solver\_adam, solver\_sgd } \\ \cmidrule{4-6}
& & &Random Forest &3 &min\_samples\_split, criterion\_entropy, criterion\_gini \\ \cmidrule{4-6}
& & &lSVM &1 &C \\ \cmidrule{4-6}
& & &kNN &2 &n\_neighbors, p \\ \cmidrule{4-6}
& & &Decision Trees &1 &min\_samples\_split \\ \cmidrule{4-6}
& & &AdaBoost &2 &learning\_rate, n\_estimators \\ \cmidrule{4-6}
& & &GaussianNB &1 &- \\ \cmidrule{4-6}
& & &Perceptron &1 &- \\
\bottomrule
\end{tabular}
\end{table*}

\begin{table*}[!htp]\centering
\caption{Search Space for ZAP Meta-Dataset}\label{tab:search_space_zap}
\scriptsize
\begin{tabular}{lllllll}\toprule
\textbf{Stage} &\textbf{$L_i$} &\textbf{$M_i$} &\textbf{Algorithm} &\textbf{$|\Lambda_{i,j}|$} &\textbf{Hyperparameters} \\\midrule

\multirow{4}{*}{Architecture} &\multirow{4}{*}{1} &\multirow{4}{*}{4} &ResNet &1 & IsActive\\ \cmidrule{4-6}
& & &EfficientNet-B0 &1 & IsActive \\ \cmidrule{4-6}
& & &EfficientNet-B1 &1 & IsActive \\ \cmidrule{4-6}
& & &EfficientNet-B2 &1 & IsActive \\ \midrule
\multirow{1}{*}{Common Hyperparameters} &\multirow{1}{*}{31} &\multirow{1}{*}{1} & - &31 &  \makecell[l]{early\_epoch, first\_simple\_model, \\ max\_inner\_loop\_ratio, \\ skip\_valid\_score\_threshold, test\_after\_at\_least\_seconds, \\ test\_after\_at\_least\_seconds\_max, \\ test\_after\_at\_least\_seconds\_step,\\ 
batch\_size, cv\_valid\_ratio, max\_size, \\ max\_valid\_count, steps\_per\_epoch,  \\ train\_info\_sample, \\ optimizer.amsgrad, optimizer.freeze\_portion, optimizer.lr, \\ optimizer.min\_lr, optimizer.momentum, optimizer.nesterov, \\ optimizer.warm\_up\_epoch, \\ warmup\_multiplier, optimizer.wd, \\ simple\_model\_LR, simple\_model\_NuSVC, simple\_model\_RF, \\ simple\_model\_SVC, optimizer.scheduler\_cosine, \\ optimizer.scheduler\_plateau, \\ optimizer.type\_Adam, \\ optimizer.type\_AdamW } \\
\bottomrule
\end{tabular}
\end{table*}

\newpage
\section{Meta-Dataset Splits}
\label{appendix:meta_dataset_splits}

We specify the IDs of the task used per split. The ID of the tasks are taken from the original meta-dataset creators.

\textbf{(i) PMF Meta-Dataset} 

\textbf{Meta-training:} 4538, 824, 1544, 1082, 1126, 917, 1153, 1063, 722, 1145, 1106, 1454, 4340, 477, 938, 806, 866, 333, 995, 1125, 924, 298, 755, 336, 820, 1471, 1120, 1520, 1569, 829, 958, 997, 472, 1442, 1122, 868, 313, 928, 921, 1446, 1536, 1025, 4534, 480, 723, 835, 1081, 950, 300, 1162, 821, 469, 933, 343, 766, 936, 1568, 785, 31, 164, 395, 761, 1534, 1056, 685, 1459, 230, 867, 828, 161, 742, 1136, 385, 877, 11, 1066, 1532, 1533, 941, 468, 1542, 795, 329, 792, 782, 1131, 796, 4153, 448, 1508, 1065, 1046, 1014, 54, 780, 748, 1150, 793, 1441, 1531, 717, 819, 1151, 287, 1016, 4135, 874, 162, 1148, 1005, 956, 1528, 23, 1516, 446, 1567, 41, 729, 910, 1156, 32, 1041, 1501, 955, 1129, 827, 937, 180, 1038, 973, 36, 44, 1496, 855, 400, 754, 1557, 1413, 758, 817, 1563, 181, 1127, 43, 444, 277, 1141, 715, 725, 884, 790, 880, 853, 155, 223, 1529, 1535, 6, 1009, 744, 1107, 1158, 830, 859, 947, 1475, 813, 734, 976, 227, 1137, 762, 777, 751, 784, 886, 885, 843, 1055, 1486, 1237, 225, 39, 778, 721, 392, 312, 857, 457, 1450, 209, 779, 479, 718, 801, 770, 1049, 391, 12, 730, 759, 1013, 338, 719, 988, 974, 787, 60, 741, 865, 1050, 735, 1079, 1482, 1143, 954, 1020, 1236, 814, 1048, 892, 879, 745, 971, 913, 1152, 694, 1133, 765, 905, 804, 848, 40477, 846, 334, 791, 923, 377, 1530, 889, 1163, 1006, 749, 922, 10, 59, 1541, 310, 461, 1538, 398, 870, 1481, 970, 1036, 1044, 1068, 187, 476, 1157, 40478, 1124, 1045, 845, 62, 915, 1167, 1059, 458, 815, 28, 797, 462, 21, 952, 467, 1505, 375, 882, 1011, 1460, 964, 1104, 275, 732, 189, 478, 1464, 979, 40474, 772, 720, 1022, 823, 811, 463, 61, 1451, 1067, 1165, 184, 716, 962, 978, 916, 1217, 935, 900, 925, 919, 871, 808, 335, 1457, 799, 983, 1169, 1004, 837, 1507, 4134, 890, 1062, 1510, 818, 728, 1135, 1147, 1019, 450, 1561, 40476, 816, 1562, 740, 864, 942, 151, 713, 953, 737, 1115, 1123, 1545, 1498, 850, 873, 959, 951, 987, 991, 1132, 1154, 294, 1040, 894, 26, 878, 307, 881, 746, 679, 872, 863, 943, 18, 1537, 767, 794, 1121, 1448, 401, 14, 1026, 833, 875, 1488, 383, 914, 20, 1043, 1116, 292, 847, 1540, 1069, 1155, 1015, 1238, 1149, 1546, 841, 1565, 1556, 1527, 682, 465, 1144, 769, 1517, 756, 834, 912, 807, 904, 16, 1061, 386, 805, 3, 775, 464, 50, 1455, 1021, 1160, 1140, 1489, 1519, 946, 994, 46, 22, 1443, 339, 969, 965, 30, 977, 860, 1500, 1064, 776, 822, 182, 743, 934, 1060, 803, 980, 1539, 346, 788, 1444, 1467, 727, 1509, 903, 832. 

\textbf{Meta-Test:} 906, 789, 1159, 1600, 48, 1453, 876, 929, 1012, 891, 1164, 726, 459, 37, 812, 909, 927, 774, 278, 279, 1054, 918, 763, 394, 948, 40, 1100, 736, 1503, 1071, 1512, 1483, 53, 869, 285, 773, 1518, 197, 926, 836, 826, 907, 920, 1080, 1412, 276, 764, 945, 1543, 1472, 996, 908, 896, 851, 397, 783, 1084, 731, 888, 733, 1473, 753, 683, 893, 825, 902, 750, 1078, 8, 1073, 1077, 475, 724, 1513, 384, 388, 887, 714, 771, 1117, 1487, 337, 1447, 862, 838, 949, 800, 931, 911. 

\textbf{Meta-Validation:} 1075, 747, 901, 1452, 389, 387, 752, 932, 768, 40475, 849, 1564, 1449, 895, 183.

\textbf{(ii) TensorOBOE Meta-Dataset}

\textbf{Meta-Training} 210,  20, 491, 339,  14, 170, 483, 284, 543, 220, 493,  64, 524,
       485, 120,  81, 495, 362, 243, 545, 538, 532, 160, 541, 238, 436,
       320, 272, 497, 412,  51, 195, 191, 116, 345, 400, 164, 106, 376,
        63, 105, 308, 523, 490, 319,  93, 468, 517, 198, 145, 150,  39,
       502, 364, 253, 303, 471,   2, 221, 518, 146, 241, 457, 114, 372,
       176, 168, 536, 350, 338, 136, 416, 254, 337, 311, 464, 424, 255,
       232, 133,  33,  88, 290,  44,  61, 199, 492, 529, 500, 343, 218,
       302, 297,  73, 295,  35, 344,  29, 432, 410, 417, 309, 527, 217,
        27, 402, 351, 156, 403, 414, 138, 212, 104, 438, 415, 421, 215, 466, 189, 214, 508, 204, 234,
       259,  67,  24, 216, 300, 223, 129, 458, 111, 166, 505, 477,  40,
       274, 427,  79, 375, 380, 327,  13, 287, 326, 496, 251, 228, 420,
       161,  83, 117,  25, 110, 149, 152,  16, 407, 331, 109, 441, 422,
       139, 237, 260, 352, 428, 317, 323, 484, 248, 449, 467,  19, 328,
       296, 454, 269, 363, 226, 465,   3, 542, 125, 280, 286,  77, 184,
       371, 455, 540, 275, 294, 521, 182,  32,  80, 307, 258,  11, 360,
       447,  86, 266,  36, 193,  58,  41, 270, 411,  50, 209, 481, 480,
       504, 503, 123, 222, 419,  62, 456, 377, 130, 187,  23, 451, 479,  43, 370, 394,   0, 383, 201,
       405, 368, 515,  98, 387, 349, 304, 418, 292, 178, 369, 256,  94,
       197,  95, 535, 163, 169,  69, 305,  48, 341, 373, 397, 207, 279,
       514, 227, 148, 143, 334, 180, 356, 460, 131, 127,  47, 452, 262,
       324, 203,  84, 426, 121, 544, 520, 534, 398, 384,  91,  82, 430,
       267, 119, 358, 291,  57, 425, 487, 321, 257, 442,  42, 388, 335,
       273, 488,  53, 522, 128,  28, 183, 459, 510, 151, 244, 265, 288,
       423, 147, 177,  99, 448, 431, 115,  72, 537, 174,  87, 486, 314,
       396, 472,  70, 277,   9, 359, 192

\textbf{Meta-Test} 118, 159, 548, 453, 385,  31, 512, 353, 247, 179, 332, 379,  10,
       489, 112, 293, 219, 395, 281,  65, 409, 126, 401, 526, 342, 346,
       413, 137, 366,   7, 381, 506, 289, 539, 282, 101,  97, 278,  54,
        30, 298,  49, 100, 474, 461, 322, 283,  56, 144,  60,   6,   8,
       507, 310, 336, 225, 261,  38, 329, 365, 445, 429, 513, 188, 469,
       124, 154, 340,  59, 312, 473, 498, 546, 528, 263, 194,  55, 171,
       236, 206, 158, 196,  34, 408,  18, 501, 250, 533,  52,  74,  26,
       173,  92, 167,   4, 382, 181, 208, 354, 249, 450,   5, 141, 525,
       200, 135, 531, 122,  22,  68

\textbf{Meta-Validation} 85, 446,  96, 172, 134,  37, 392,  90, 509, 389, 378, 435,  66,
       391, 530, 333, 462, 231, 330, 301, 325, 268, 434, 318, 233, 213,
       549, 140, 264, 482, 155, 235, 175, 157, 113, 165, 245, 246,  15,
       361, 547, 470,  17, 306, 190, 153, 357,  45, 443, 162, 475, 186,
       224, 494, 393, 399, 444, 550, 439, 516, 433, 230, 108,  89, 406,
        46, 102, 463,  21, 107, 374, 211, 103,  71,  75, 316,  78, 240,
       205, 386, 202, 142, 313, 252, 348, 511, 437, 347, 478, 355, 476,
       242, 276, 519, 499, 285, 271, 229,   1, 390,  12, 132, 299, 404,
       440, 239, 185,  76, 367, 315

\textbf{(iii) ZAP Meta-Dataset}

\textbf{Meta-Train} 0-svhn\_cropped, \\
svhn\_cropped,
cycle\_gan\_apple2orange,
cats\_vs\_dogs,
stanford\_dogs,
cifar100,
coil100,
omniglot,
cars196,
cars196,
horses\_or\_humans,
tf\_flowers,
cycle\_gan\_maps, 
rock\_paper\_scissors,
cassava,
cmaterdb\_devanagari,
cycle\_gan\_vangogh2photo, \\
cycle\_gan\_ukiyoe2photo,
cifar10,
cmaterdb\_bangla,\\
cycle\_gan\_iphone2dslr\_flower,
emnist\_mnist,
eurosat\_rgb,
colorectal\_histology,
cmaterdb\_telugu,
uc\_merced,
kmnist,

\textbf{Meta-Test} 0-cycle\_gan\_summer2winter\_yosemite, \\
cycle\_gan\_summer2winter\_yosemite,
malaria,
cycle\_gan\_facades,
emnist\_balanced,
imagenette,
mnist,
cycle\_gan\_horse2zebra

\textbf{Meta-Validation} 
emnist\_byclass,
imagenet\_resized\_32x32,
fashion\_mnist

\textbf{(iv) OpenML Datasets} 

10101, 12, 146195, 146212, 146606, 146818, 146821, 146822, 146825, 14965, 167119, 167120, 168329, 168330, 168331, 168332, 168335, 168337, 168338, 168868, 168908, 168909, 168910, 168911, 168912, 189354, 189355, 189356, 3, 31, 34539, 3917, 3945, 53, 7592, 7593, 9952, 9977, 9981

We checked that there is not overlap between the tasks used for meta-training from the TensorOBOE and the tasks used on OpenML Datasets.

\end{document}